\documentclass[10pt]{article} 
\usepackage[preprint]{tmlr}

\usepackage{url}

\usepackage{microtype}
\usepackage{graphicx}
\usepackage{subcaption}
\usepackage{booktabs} 
\usepackage{tabularx}
\usepackage{array}
\usepackage[table]{xcolor}

\usepackage{hyperref}

\usepackage{pifont}
\usepackage{amsmath}
\usepackage{amssymb}
\usepackage{mathtools}
\usepackage{amsthm}

\usepackage[capitalize,noabbrev]{cleveref}

\theoremstyle{plain}
\newtheorem{theorem}{Theorem}[section]

\newtheorem{lemma}[theorem]{Lemma}

\theoremstyle{definition}
\newtheorem{definition}[theorem]{Definition}

\theoremstyle{remark}

\newtheorem{notation}{Notation}

\usepackage{tikz}
\usetikzlibrary{
    arrows.meta,
    positioning,
    fit,
    calc,
    backgrounds,
    shapes.geometric
}

\tikzset{
    rv/.style={
        circle,
        draw=black,
        thick,
        fill=white,          
        minimum size=8mm,    
        inner sep=2mm
    },
    kernel/.style={
        rectangle,
        draw=black,
        thick,
        fill=white,
        minimum size=8mm,    
        inner sep=2mm
    }
}

\usepackage{enumitem}

\newcommand{\R}{\mathbb{R}}

\newcommand{\Y}{\mathcal Y}
\newcommand{\X}{\mathcal X}
\newcommand{\T}{\mathcal T}
\newcommand{\ZZ}{\mathcal Z}
\newcommand{\V}{\mathcal V}

\newcommand{\RR}{\mathcal R}

\newcommand{\rv}[1]{\mathsf{#1}}
\newcommand{\p}[1]{p_{#1}}
\newcommand{\EV}[2]{\mathbb{E}_{#1}\left[#2\right]}

\newcommand{\ProbDists}[1]{\mathcal{P}\left(#1\right)}
\newcommand{\ProbOf}[2]{P_{#1}\left(#2\right)}
\newcommand{\fnset}[2]{{#2}^{#1}}
\newcommand{\proc}{\mathtt}

\newcommand{\gv}[2]{\left. #1 \middle | #2 \right.}

\newcommand{\pushforward}[2]{{#1}_{\sharp}{#2}}
\newcommand{\kernelover}[2]{{#1}\rightsquigarrow{#2}}
\newcommand{\MarkovKernels}[2]{\mathcal{M}\left(#2\middle | #1\right)}

\newcommand{\Err}[2]{\mathrm{Err}_{#1}\left(#2\right)}

\newcommand{\deficiency}[3]{\delta_{#1}\left({#2}; {#3}\right)}
\newcommand{\divergence}[3]{d_{#1}\left({#2}; {#3}\right)}
\newcommand{\condDeficiency}[4]{\delta_{#1}\left({#2}; {#3} \middle | {#4} \right)}
\newcommand{\condDivergence}[4]{d_{#1}\left({#2}; {#3} \middle | {#4} \right)}
\newcommand{\bregman}[4]{B_{#1}\left({#2}; {#3}; {#4} \right)}
\newcommand{\condBregman}[5]{B_{#1}\left({#2}; {#3}; {#4} \middle | {#5} \right)}

\newcommand{\ux}{v}
\newcommand{\UX}{\V}
\newcommand{\uy}{z}
\newcommand{\UY}{\ZZ}
\newcommand{\duy}{w}
\newcommand{\fapprox}[1]{\hat{#1}}
\newcommand{\Dapprox}[1]{\tilde{#1}}

\newcommand{\representationErr}{\Delta_\rho}
\newcommand{\uDef}{\Delta_u}
\newcommand{\fDef}{\Delta_f}
\newcommand{\excessFDeficiency}{\Phi_f}
\newcommand{\fModelErr}{\varepsilon_f}
\newcommand{\dErrU}{\Delta_d^u}
\newcommand{\dErrF}{\Delta_d^f}
\newcommand{\excessdErr}{\Phi_{d}}

\newcommand{\DErrU}{\Delta_{\proc D}^{u}}
\newcommand{\DErrF}{\Delta_{\proc D}^{f}}
\newcommand{\excessDErr}{\Phi_{\proc D}}

\newcommand{\BregmanF}{B_{\proc{D}}^f}

\newcommand{\rhoD}{{\rho_{\proc D}}}

\newcommand{\rE}{{r_{\proc E}}}
\newcommand{\rD}{{r_{\proc D}}}

\newcommand{\inflationTerm}{\iota_{\ell \leftarrow \ell_u}}
\DeclareMathOperator*{\esssup}{ess\,sup}

\newcommand{\Dstar}{\proc{D}^{\star}}
\newcommand{\Ybar}{\mathcal Y}

\usepackage[textsize=tiny]{todonotes}

\title{A Theoretical Framework for Understanding and Building ``Any-task inference'' Capable Foundation Models}

\author{%
    \name Matthew McDermott \email mm6677@cumc.columbia.edu \\
    \addr Columbia University Department of Biomedical Informatics
    \AND
    \name Florent Pollet \\
    \addr Columbia University Department of Biomedical Informatics
    \AND
    \name Payal Chandak \\
    \addr Harvard Medical School
    \AND
    \name Gregory Kondas\\
    \addr Columbia University Department of Biomedical Informatics%
    \AND Liat Friedman Antwarg\\
    \addr Harvard Medical School%
    \AND Noa Dagan\\
    \addr Clalit Health Services and Ben Gurion University%
    \AND Isaac Kohane\\
    \addr Harvard Medical School%
}

\def\month{MM}  
\def\year{YYYY} 
\def\openreview{\url{https://openreview.net/forum?id=XXXX}} 

\begin{document}

\maketitle

\begin{abstract}
In this work, we provide a formal framework and mathematical definition for the subset of foundation models focused on enabling the property of ``any-task inference'' (ATI) (similar to zero-shot inference), rather than those designed for pre-training fine-tuning. 
To achieve this mathematical definition, we formalize the notion of a ``space of tasks of interest,'' which we include in the problem domain, then provide an interface-based formal definition for ATI foundation models as those that map samples of data $x$ and tasks $t$ to predictions over the space of labels directly.
While simple, this definition allows us to show that all ATI foundation models can be decomposed through a ``universal task reduction'' lens into encoder, model, and decoder aspects. We can leverage this decomposition to also decompose foundation model error into terms corresponding to different aspects of this reduction, stratified into a collection of ``design terms,'' which are theoretical properties that can inform one of whether a particular proposed ATI foundation model can achieve optimality or not due to the structure of the proposed task reduction, and ``realized model excess'' terms which capture interactions between learned and programmatic components of foundation model systems.
Lastly, we can use this perspective to offer two ``prototypical'' recipes for realizing ATI foundation models in diverse domains, which we term generative and discriminative variants of ATI foundation models.
This exercise is grounded throughout by 5 real-world examples of foundation models in different settings, varying greatly in how these capabilities are realized, showing the applicability of this definition to real-world settings.
\end{abstract}

\section{Introduction}
\label{sec:introduction}
This work aims to provide a formal definition of the subset of ``foundation models'' that are designed to enable the practical property typically called ``zero-shot inference,'' which we recognize more directly as the ability to use the foundation model to perform inference for any arbitrary task in a given domain, on any input sample, with no additional training required. A formal definition for this class of foundation models will reduce ambiguity when this term is used in research, allow for more comprehensive theoretical understanding of these foundation models, and help us see how we can build them in diverse domains. 

\subsection{Focusing on Foundation Models that enable Direct ``Any-task inference.''}
The term ``foundation model'' was originally introduced by \citet{crfm_fms} as
``\emph{models ... trained on broad data (generally using self-supervision at scale) that can be adapted to a wide range of downstream tasks}''~\citep{crfm_fms}. This definition has since been applied to many different kinds of models that share different traits and usage goals. In this work, we assume that the key property in this definition, and the property that matters most for something to be a ``foundation model'' is that of its data-efficient utility on a ``wide range of downstream tasks.''
In the modern era, this utility is best epitomized through models that permit \emph{immediate, no-fine-tuning required inference} across an \emph{arbitrary range of reasonable tasks over a domain of interest.} We will use the term ``any-task inference (ATI) capability'' to refer to this ability -- being able to immediately perform inference on an arbitrary task of interest without any additional fine-tuning.\footnote{We use this term rather than the term ``zero-shot inference'' because our primary interest is not in \emph{how a model was trained} (i.e., what particular style of learning objective was used) but rather in \emph{what you can do with it at inference time} (i.e., make predictions for arbitrary tasks of interest), though in many cases these two capabilities will intersect.} 

In this work, we will focus our analyses specifically on foundation models that are capable of ATI, which we will call ATI foundation models, rather than other kinds of foundation models. In so doing, we are explicitly not studying other notions of what may make something a ``foundation model'' that are commonly used today. In particular, in this work we will take the following positions:
\begin{enumerate}
    \item We will assume that foundation models need not be generative models. Instead, the focus on downstream task utility is paramount, not their ability to estimate the likelihood of or generate samples from the data domain.
    \item We will assume that foundation models need not be trained with any particular style of self-supervision or neural network architecture. What makes foundation models unique is how they can be used at inference time, not how they were trained.
    \item We will focus our analyses on ATI foundation models, \emph{rather than foundation models that require fine-tuning and do not permit ATI.}
    \item We will assume that foundation models need not solve problems of distribution shift over their domains, provided they still show any-task adaptability to diverse downstream tasks in the domain of interest. Critically, as we will formalize later in the work, any-task adaptability itself need not require distribution shift, so this is not in conflict.
\end{enumerate}

Note that we do not argue that these topics are not interesting, useful, and important in their own right -- just that we can separate them from the concept of an ``ATI foundation model,'' (and arguably from the term ``foundation model'' more generally) and reserve this new term to refer more specifically to these any-task capable models. The established terms for each of these capabilities (e.g., generative modeling, self-supervised modeling, transfer learning, and distribution shift) can instead be used to refer to those specific areas individually. Note that we may at times just use the term ``foundation model'' in the text, but we will near universally use this to mean an ATI foundation model unless otherwise specified.

\subsection{This Definition Permits Meaningful Theoretical Contributions}
With this narrowed focus, we can provide a formal mathematical definition of ATI foundation models that permits us to make several theoretical advances. Concretely, to summarize the contributions of this work:
\begin{enumerate}
    \item We re-introduce the notion of a ``distribution of tasks'' over a space $\T$ that is jointly distributed with one's data over $\X$ and labels over $\Y$.\footnote{To the best of our knowledge, this concept was first introduced in~\cite{baxter}.} We further introduce the notion of stochastic representations $\rho$ of tasks that map from the (potentially intractable and unobserved) space $\T$ to a usable space for modeling.
    \item Given this joint domain, we define any-task inference capable foundation models as \emph{procedures} that operate jointly over pairs of data samples and tasks to predict the target label: $\proc{P}: \X \times \T \to \ProbDists{\Y}$.
    \item We show that any foundation model under this definition can be decomposed into ``encoder'', ``model'', and ``decoder'' stages, which formalize the view that foundation models enable any-task inference by performing an algorithmic reduction in the space of tasks to some ``universal task.''
    \item We further show that this decomposition gives rise to a clear decomposition of the sources of error any foundation model has, across terms dependent on each of these components, including the following error terms:
    \begin{itemize}
        \item Irreducible additional uncertainty due to low fidelity representations of the input task.
        \item Limitations in how effectively the universal task can approximate the true task.
        \item Error in how effectively the model component actually models the universal task.
        \item Error arising from the choice of labeling function used to translate universal-task outputs into downstream labels.
        \item Error arising from imperfect approximation of the target decoding distribution by the realized decoder.
    \end{itemize}
    \item We give two practical recipes for how to build foundation models in diverse domains.
    \item We show that foundation model evaluation is mathematically analogous to evaluation in the context of algorithmic fairness, yielding identical impossibility criteria. 
\end{enumerate}

\subsection{Running examples}
We will use several running examples of ATI foundation models in the literature today throughout this work, including the following:
\paragraph{Large language models}
We will consider large language models (LLMs) as ATI foundation models in the sense that, through prompting and/or specific decoding strategies, diverse (and arguably all) tasks in natural language processing (NLP) can be reframed as next-token prediction problems, and thus a well trained LLM arguably should be useful for solving arbitrary examples of such tasks in NLP. Examples include~\cite{GPT-2,GPT-3,GPT-4}. We separate these any-task capable examples from PT/FT LLMs such as BERT~\citep{BERT}, which we will not consider here (though we note that masked language models can in theory be used for any-task inference as well, though this is less well explored in practice).

\paragraph{Autoregressive Foundation Models for Electronic Health Record Data}
An emerging area of foundation models today concerns autoregressive (AR) sequence models applied over longitudinal electronic health record (EHR) datasets (typically expressed in event-stream format, e.g., in the MEDS standard~\citep{MEDS} which is then tokenized into some categorical sequence~\citep{tokenization_guo_fries_sung,tokenization_lee_solo_bealieu-jones,tokenization_rafi_rajna}). These models appear superficially similar to LLMs, but perform any-task inference in a distinct pathway. Whereas LLMs use prompting to modify the ``input'' for an any-task problem in some controlled fashion, EHR AR FMs do not modify the input (a stream of patient level data) at all (and are generally not capable of prompting). Instead, they simply forecast possible ``simulated futures'' of the patient data, then infer the task label via some statistical form (e.g., Monte Carlo sampling) over those simulated futures directly. Examples include~\cite{ETHOS,curiosity}.

\paragraph{Prior Fitted Networks}
Another emerging area of foundation models concerns prior fitted networks~\citep{PFNs}, which frame foundation models as a Bayesian inference problem to learn arbitrary posteriors given samples of that posterior learned through a dataset derived from a prior over distributions. Recent works such as TabPFN~\citep{tabpfn} have realized that these models can be used as powerful foundation models over, e.g., tabular data, though further recent work has also applied this method to other data modalities, such as imaging data~\citep{woerner2026universalalgorithmimplicitlearning}. 
Unlike other foundation model variants, the input for a PFN looks quite different from a traditional model that maps $x$ to $y$. In particular, PFNs explicitly take as input a sequence of sampled pairs $(x, y)$ that represent the task of interest at inference time, followed by a final, unpaired $x$ sample for which the target label $y$ should be inferred. During pre-training, these sampled sequences are generated from a defined ``prior'' over possible tasks, enabling many diverse relationships between $x$ and $y$ to be learned. This makes PFNs extremely flexible, but also imposes computational resource limitations on when they can be used given a sequence of many pairs $(x, y)$ must fit on the GPU for a single sample. While one may argue that this sequence of $(x, y)$ pairs used as input to the model violates the ``any-task'' contract of interest here, we will permit this, and justify it more formally later in the work.

\paragraph{Task-conditioned or Multi-task Foundation Models over an Explicit Task Space}
Classically, multi-task learning~\citep{ehrpt} or even single-task (but typically multi-class) supervised pre-training have been widely used in representation learning settings~\citep{imagenet}. More recently, many variants of multi-task pre-training algorithms have been presented as ``foundation models'', in particular in health domains~\citep{motor,ora}. A superficially related approach, in which task-conditioned labels are produced, rather than producing labels for multiple tasks simultaneously, has also recently been proposed as an alternative to AR foundation models in the health domain~\citep{everyquery}. 

In multi-task models, a model outputs a separate prediction for a finite set of tasks simultaneously. In task-conditioned models, the model takes as input a sample $x$ but also some ``representation of the task of interest'' in some structured form, then outputs a prediction solely for that task and input. In this way, it takes many inference steps for a task-conditioned model to produce the same output as a multi-task model, but the task-conditioned model can model a much larger (potentially infinite) set of tasks simultaneously, in contrast to the finite set of tasks expressible by a multi-task model. These models realize ``any-task'' capabilities over the structured space of tasks over which labels are produced. If either this space is sufficiently broad or a procedure on how to adapt other tasks into this structured space of tasks is provided, it is reasonable to consider these models as ``foundation models'' in our more restricted sense (though we note that often these models are presented in a manner requiring PT/FT or transfer learning).

A summary of the key characteristics of these examples is provided in Section~\ref{sec:examples} and Table~\ref{tab:example_fm_decompositions}.

\subsection{Outline of the Rest of the Work}

Our notation uses precise conventions for typeface and notational elements in several places: for example, we write random variables in sans-serif typeface and distinguish them from their values, write Markov kernels as $\rho : \mathcal{T} \rightsquigarrow \mathcal{R}$, and use push-forward distributions $f_\sharp q$. These conventions, together with the standing assumptions we rely on, are collected in Appendix~\ref{sec:notation}. \textbf{We strongly recommend that readers review Appendix~\ref{sec:notation} before proceeding to the rest of the work, as the formal development begins there and uses this notation without further introduction.}

In the rest of this work, we will first formalize the notion of the ``distribution of tasks'' and of ``task representations'' in Section~\ref{sec:task_space_and_representations}. Then, in Section~\ref{sec:defining_foundation_models}, we will formally define ATI foundation models
and demonstrate that this definition universally permits a decomposition through a ``universal task reduction'' lens. Building on this definition, in Section~\ref{sec:FM_error_sources}, we will decompose the error of any given foundation model into meaningful terms dictated by the universal task decomposition. In Section~\ref{sec:FM_recipes}, we will further show that we can define two concrete recipes to build foundation models, building on the same basic assumptions and describe how they translate error across these different error sources. Finally, in Section~\ref{sec:FM_evaluation}, we will discuss evaluating foundation models, and show that it is mathematically equivalent to evaluation of supervised models under fairness constraints.

\section{Formalizing Tasks in Our Notation and Problem Formulation}
\label{sec:task_space_and_representations}
A key ingredient of our eventual ATI foundation model definition will be the elevation of a ``task'' to a part of the input to the procedure, and thus to the probabilistic formulation of the problem. In this section, we provide the necessary mathematical and notational set-up to do this formally.

\subsection{The Extended Task-Data Domain}
Let $\X$ be some data domain. Let $\Y$ be a label domain, such as $\{0, 1\}$ for binary classification, though $\Y$ need not be discrete.
Typically, in machine learning, we assume that we have random variables $\rv x, \rv y$ over $\X \times \Y$ that define our space of interest. 
In this work, we argue this view is insufficient. As was first introduced (to the best of our knowledge) by \citeauthor{baxter}, we argue that we should also include a notion of a standard Borel space $\T$ of ``tasks'' requiring us to predict $\Y$ from $\X$\footnote{Note this relies on assuming that all tasks share the same label space -- we will revisit this later but it is a more comfortable assumption than it seems at first glance.}. This notion requires us to similarly extend our (joint) distribution to $\rv x, \rv y, \rv t$ over data inputs $\rv x$, labels $\rv y$ and the \emph{task to which that label corresponds} $\rv t$. We will explicitly assume that this joint random variable is well defined for any valid domain, and further that the random variable $\rv t$ always has full support over $\T$.


To make our notion of a task $t \in \T$ clear, we will identify a task $t$ with its associated posterior distribution of labels given data by assuming that there exists a jointly measurable kernel $\eta: \kernelover{\X \times \T}{\Ybar}$, such that $\gv{\rv y}{\rv x = x, \rv t = t} \sim \eta(\cdot | x, t)$.
We will further assume that the mapping $t \mapsto \eta_t = \eta(\cdot | \cdot, t)$ is injective -- i.e., each task defines a unique posterior distribution of $\gv{\rv y_t}{\rv x}$. Note that while, mathematically, we assert that $\eta$ exists, we are not claiming that it is ever known -- indeed, if $\eta$ were known, we would not need to do any learning as it defines the true posterior distributions.

\paragraph{Distribution Shift in the Task Space}
Note that formalizing this space of tasks $\T$ and distribution over tasks $\rv t$ implies we may observe distribution shift over tasks between training and evaluation. This is an important area of study, but not, we argue, essential to ``any-task inference'' capabilities or to ``foundation models,'' and we leave it as an area of future work.

\paragraph{Existence of Trivial Task Distributions}
Note that we are not claiming that there is only ``one'' true task space $\T$ or task distribution $\rv t$. Instead, just as we have many data domains $\X$ or data distributions, we can have many task domains or distributions, and define new ones to suit certain applications or contexts. This does mean that we can define some ``trivial'' task spaces -- e.g., a single element task distribution $\T= \{t\}$. This is valid in much the same way that it is valid to define my data domain as a single point $\X = \{x\}$. It is technically true, but neither interesting nor useful. Given this flexibility, specifying $\T$ is therefore an important part of specifying the domain of interest for any claims made about associated foundation models.

\paragraph{Single-task Models in this Extended Domain}
Note that this formulation does not preclude us from studying single-task models for a given task $t \in \T$. In this setting, we are interested solely in modeling $\rv x_t, \rv y_t$ (in either a generative or discriminative $\gv{\rv y_t}{\rv x}$ setting). In such a setting, we can safely discard $t$ entirely, define our ``effective'' variables $\rv x, \rv y$ to be given by $\rv x_t, \rv y_t$, and model as normal without issue. 

\paragraph{Any-task Inference Capabilities in this Extended Domain}
If we are not interested in a single-task model, however, and instead interested in models that display ``any-task inference capabilities,'' then this extended domain is precisely the machinery we need to formalize this capability. In particular, for a model to support any-task inference over $\X, \T, \Ybar$ it is clear that this model must permit, through some procedural form, one to map \emph{any} pair $(x, t)\in \X \times \T$ to some prediction distribution over $\Ybar$ -- e.g., the model must be usable through a procedure that lives in $\fnset{(\X \times \T)}{\ProbDists{\Ybar}}$.

Critically, note here that we are not asserting anything about the model's actual performance. It may, for example, yield terrible performance over every task $t \in \T$, or be constant in its predictions across all $t \in \T$, and that need not violate its ability to perform ``any-task inference.'' Instead, the claim we make here is operational -- if someone describes their model as being able to make inference over any task within this extended domain, then there must be a mechanism by which you can perform inference for any task $t \in \T$, which we can formalize in this functional form.

Note that we are also not requiring that $\T$ be directly observable for a model to support any-task inference -- instead, the only requirement is that if you have access to a task $t \in \T$ or some representation of such a task, you must somehow be able to use the model to make inference for that task $t$. Note that this definition shares many similarities with the informal definition of ``zero-shot'' learning; for more commentary on why we use ``any-task inference'' vs ``zero-shot learning'' here, see Appendix~\ref{appendix:any-task-vs-zero-shot}.

\subsection{Task Representations}
In general, $\T$ will be an unobserved, latent space which is not represented in a tractable format. This will make it challenging to operationalize how we will realize these ``procedures'' defined on $\X \times \T$ referenced above. Instead, in practice, where ``tasks'' are used today in an operational manner, we instead must ``represent'' a task $t \in \T$ in some more tractable manner.
We formalize this through the notion of a Markov Kernel (i.e., a stochastic function) from $\T$ into a space $\RR$ via $\rho: \kernelover{\T}{\RR}$. As this kernel defines a probability distribution for each $t \in \T$, we can introduce the notion of a random variable $\rv r$, jointly distributed with $\rv x, \rv y, \rv t$ such that the conditional random variable $\rv r_t = \gv{\rv r}{\rv t = t}$ is given precisely by the action of $\rho$ on $t$. If we wish for $\rho$ to be given by a deterministic function $g:\T \to \RR$, we can represent that with a point mass distribution, such that $\rho(t) = \delta_{g(t)}$. In this case, we may (at a slight abuse in notation) use $\rho(t) = g(t)$ as well when this distinction is clear from context. Note that even in the case that $\rho$ is deterministic, $\rv r$ may still be stochastic given the stochasticity of $\rv t$. Below, we list some general properties and/or definitional requirements of task representations (all of which are assumed to be held in the context of a joint domain over $\rv x, \rv t, \rv y$). 

Given a representation $\rho: \kernelover{\T}{\RR}$, we can also define several induced kernels that will help with further results. In particular, Definition~\ref{def:posterior_task_kernel} defines $\kappa_\rho: \kernelover{\RR}{\T}$ to be the posterior kernel of tasks given representations and Definition~\ref{def:induced_kernel_phi} defines the kernel $\phi_\rho: \kernelover{\X \times \RR}{\Ybar}$ to be the distribution of labels conditioned on a data sample and a represented task (much as $\eta$ is the kernel for true tasks).

\begin{definition}
    \label{def:posterior_task_kernel}
    Given extended domain $\X, \T, \Y$, $\rho: \kernelover{\T}{\RR}$ a representation of $\T$, and associated joint random variables $\rv x, \rv t, \rv y, \rv r$, we can define the induced kernel $\kappa_\rho: \kernelover{\RR}{\T}$ to be the posterior distribution of $\rv t | \rv r$.
\end{definition}

\begin{definition}
    \label{def:induced_kernel_phi}
    Given extended domain $\X, \T, \Y$, $\rho: \kernelover{\T}{\RR}$ a representation of $\T$, and associated joint random variables $\rv x, \rv t, \rv y, \rv r$, we can define the induced kernel $\phi_\rho: \kernelover{\X \times \RR}{\Ybar}$ to be the induced distribution of $\rv y | \rv x, \rv r$.
\end{definition}

To maximize the utility of these task representations, there are some basic properties, assumptions, and definitions we will introduce.
Concretely, we take as a definitional assumption that $\rv r \perp \rv x, \rv y | \rv t$. Similarly, if we have two distinct representations $\rho_1$ and $\rho_2$, then we will likewise assume that $\rv r_1 \perp \rv r_2 | \rv t$. Given two such representations, we may also define a composite representation $\rho = \rho_1 \otimes \rho_2$ such that $\rho: \kernelover{\T}{\RR_1 \times \RR_2}$ yields a random variable $\rv r$ taking values $(r_1, r_2)$ such that each component of $\rv r_t$ is independent of one another and defined by $\rho_1(\cdot | t)$ and $\rho_2(\cdot | t)$ respectively.
Additionally, some representations $\rho$ may be ``degenerate'' in the sense that they may offer no information about the task in question, such that $\rv r \perp \rv t$. For example $\rho(t) = \delta_0$ ($\rho$ maps all tasks to a point-mass at $0$) is degenerate.

Conversely, a representation $\rho$ will be called ``complete'' if it contains all the information about the task. Concretely, $\rho$ is complete if $\kappa_\rho(\cdot | \rv r) = \delta_{\rv t}$ almost everywhere. For example, $t \mapsto \delta_t$ is complete.

A graphical model of the relationship between data, tasks, labels, and representations can be shown in Figure~\ref{fig:graphical_model_tasks}.

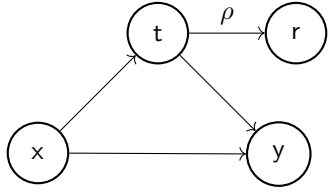
\begin{figure}
    \centering
    \begin{tikzpicture}
        \node[style=rv](x){$\rv x$};
        \node[style=rv](t)[above right=of x]{$\rv t$};
        \node[style=rv](y)[below right=of t]{$\rv y$};
        \node[style=rv](r)[right=of t]{$\rv r$};

        \draw[->] (x) -- (t);
        \draw[->] (t) -- (y);
        \draw[->] (x) -- (y);
        \draw[->] (t) -- node[midway, above] {$\rho$} (r);
    \end{tikzpicture}
    \caption{A graphical model of data samples $\rv x$, tasks $\rv t$, labels $\rv y$, and representations of tasks induced by $\rho$, $\rv r$.}
    \label{fig:graphical_model_tasks}
\end{figure}

\paragraph{Examples}
While it may seem counterintuitive at first, all task spaces permit tractable representations through a variety of means. Concretely, as we will explore in more depth below, each of our running examples leverages distinct tractable representations of tasks to achieve their foundation model behaviors. In particular, LLMs represent tasks via free-text prompts and natural language questions. AR foundation models for EHR data represent tasks in the form of structured programs they use to extract labels from generated trajectories during inference. PFNs represent a task $t$ stochastically via samples (assumed to be iid) drawn from the distribution of $\rv x, \rv y | \rv t = t$, which are directly passed to models as a sequence of $(x, y)$ pairs. Task-conditioned, direct prediction foundation models like EveryQuery \citep{everyquery} directly represent tasks via structured programs (much like AR EHR FMs use), and, finally, multi-task foundation models consider only a discrete space of tasks and represent them via categorical indices.

\subsection{Irreducible Error due to Task Representation Ambiguity}
It is important to note that $\rho$ will often produce ``lossy'' representations of tasks. In particular, it is clear by the data processing inequality and the fact that $\rv x, \rv y \to \rv t \to \rv r$ is a Markov chain, then $I(\rv x, \rv y ; \rv t) \ge I(\rv x, \rv y; \rv r)$. In other words, a representation of a task has less mutual information with the joint distribution of $\rv x, \rv y$ than does the true task.

We can generalize this notion slightly to apply to arbitrary generalized entropies and divergences by introducing the strictly proper scoring loss $\ell$ over $\Ybar$ and the associated generalized entropies introduced in Appendix~\ref{sec:notation}, through Definition~\ref{def:representation_error}.

\begin{definition}
    \label{def:representation_error}
    Given $\ell$, $\rho$ as defined above, let $\representationErr$ (implicitly dependent on $\ell$, which we will assume is specified in context) to be the increase in the generalized $\ell$-entropy $h_\ell$ of $\gv{\rv y}{\rv x}$ when conditioning on $\rv r$ rather than $\rv t$:
    \begin{align*}
        \representationErr
          &= h_\ell(\gv{\rv y}{\rv x, \rv r}) - h_\ell(\gv{\rv y}{\rv x, \rv t}) \\
          &= h_\ell(\gv{\rv y}{\rv x, \rv r}) - h_\ell(\gv{\rv y}{\rv x, \rv t, \rv r}) \\
          &= I_\ell(\gv{\rv y; \rv t}{\rv x, \rv r}) \\
          &\ge 0
    \end{align*}
\end{definition}

Definition~\ref{def:representation_error} shows that $\representationErr$ is precisely (1) the amount of information that remains between $\rv y$ and $\rv t$ when conditioning on $\rv r$ and $\rv x$, and (2) the increase in Bayes-optimal error were we forced to predict $\rv y$ from $\rv x$ and $\rv r$ alone vs. if we could use the true $\rv t$ directly.

Note that we can see that if $\rho$ is complete, then $h_\ell(\gv{\rv y}{\rv x, \rv r}) = h_\ell(\gv{\rv y}{\rv x, \rv t})$, so $\representationErr = 0$. Conversely, if $\rho$ is degenerate, then $\representationErr = I_\ell(\rv y; \rv t|\rv x)$ and is maximized. Additionally, suppose that $\rho = \rho_1 \otimes \rho_2$. Then, by independence of $\rho_1$ and $\rho_2$ and the construction of the composite representation $\rho$, we have that $\representationErr \le \Delta_\ell(\rho_1)$ and $\representationErr \le \Delta_\ell(\rho_2)$.

\paragraph{Example} Consider two tasks: predicting whether a patient will be readmitted within 7 days or within 30 days. If both are represented simply as “predict readmission,” the time horizon is lost. A readmission on day 20 would have different labels under the two tasks. Thus, even a Bayes-optimal predictor cannot fully resolve this ambiguity, resulting in irreducible representation error \(\Delta_\rho\).

\section{Defining Foundation Models}
\label{sec:defining_foundation_models}
With this machinery in place, we can now define an ATI foundation model formally. In fact, in this work, we will make the property of ``any-task inference capable'' identical with being a ``foundation model,'' as this aligns precisely with the desiderata introduced in the Introduction and the spirit of maximally data-efficient adaptability introduced by \citeauthor{crfm_fms} We will first introduce a direct definition based on any-task inference capability in Subsection~\ref{subsec:direct_def}, then further show that we can equivalently define foundation models as reductions to a universal task and establish some basic properties of this definition in Subsection~\ref{subsec:reduced_def}.

\subsection{Formal Definition of an ATI Foundation Model}
\label{subsec:direct_def}
We first define ATI foundation models formally in Definition~\ref{def:foundation_model}.

\begin{definition}[Foundation Model]
    \label{def:foundation_model}
    Let $\X, \T, \Y$ be an extended data, task, label domain of interest, satisfying the normal, previously mentioned assumptions. We will define an ``ATI foundation model'' over this domain to be a procedure\footnote{Note that we use the term ``procedure'' rather than ``function'' or ``model'' here simply to reflect that this may, operationally, involve classical, algorithmic and modeled components, not to differentiate this mathematically from any other ordinary function.} $\proc{P}: \X \times \T \to \ProbDists{\Y}$.
    Implicitly, if one is interested in an ATI foundation model over this domain, but is only permitted to access tasks through a single-draw of a representation $\rho: \kernelover{\T}{\RR}$, then we should understand this as a procedure over $\X \times \RR$ instead of $\X \times \T$ to reflect this single-draw requirement.
\end{definition}

Much as we said for ``any-task inference capability,'' Definition~\ref{def:foundation_model} does not assert a specific performance guarantee, nor does it require generalization outside of the specified domain $\X, \T, \Y$. It does, however, assert that there must be an operational manner by which one can map an input $x$ and a given task $t$ to a predicted distribution over $\Y$. Note that we also do not assert that the foundation model will be able to access the raw task $t$ -- instead, as we will introduce later, we will assume there is a specific representation of the task that the foundation model can leverage, which will factor into eventual error decompositions.

While Definition~\ref{def:foundation_model} may seem weak, we will see in this section that it permits meaningful theoretical and practical results. Concretely, in what remains of this section we will first remark that this definition trivially demonstrates what the ``Bayes-optimal'' foundation model is over this domain, and the implications of this, then we will show that all foundation models under this definition can be decomposed into terms under a ``universal-task reduction'' lens.

\paragraph{Bayes-optimal Foundation Models}
It is clear from Definition~\ref{def:foundation_model} that over a given domain with joint random variables $\rv x, \rv t, \rv y$, the Bayes-optimal foundation model $\proc P^*$ under any strictly proper scoring loss $\ell$ is given by the true posterior of $\rv y$ given $\rv x, \rv t$:

\[\proc P^*: (x, t) \mapsto \eta(x, t).\]

We can also use this to see, near trivially, that the best possible foundation model cannot outperform the best possible single-task model for a task $t \in \T$, both in general and when under matched model-permissibility constraints, thereby formally clarifying that foundation model improvements over supervised models must come from other mechanisms (Appendix~\ref{appendix:fms_v_supervised_learning}).

\paragraph{Foundation Models over Represented Task Spaces}
Definition~\ref{def:foundation_model} defines a foundation model over an extended $\X, \T, \Ybar$ domain and over the $\X, \RR, \Ybar$ space where $\RR$ is the representation space of some representation $\rho: \kernelover{\T}{\RR}$ in the case that the user is only allowed to access a single-draw of $\rho$ for any task $t \in \T$. Note that even in this single-draw case, we still consider the foundation model to be ``over the $\X, \T, \Y$ domain'' not for the purposes of its definition/API, but instead for the purposes of optimality, as were the user to use a different representation $\rho': \kernelover{\T}{\RR'}$ for the root space of $\X,\T,\Y$, they might be able to obtain a ``better'' model than is obtainable under $\rho$ in a manner that is lost if we consider optimality to be in the $\X, \RR, \Y$ space. We will show later that the cost of a given choice $\rho$ in this single-draw case is precisely the irreducible representation error $\representationErr$ as defined in Definition~\ref{def:representation_error}.

\paragraph{Example} A multi-task model predicting a fixed set of clinical outcomes satisfies Definition~\ref{def:foundation_model} when $\T$ is restricted to those outcomes: given a task index, the corresponding prediction can be returned without retraining. However, this guarantees no inference capability beyond the specified task space, emphasizing that the expressiveness of $\T$ is central to interpreting foundation model claims.

\subsection{Foundation Models as Universal Task Reductions}
\label{subsec:reduced_def}
Building on only Definition~\ref{def:foundation_model}, we can decompose any foundation model into a procedure that leverages a form of \emph{reduction} in the space of tasks to achieve any-task inference capabilities. Concretely, we will break a foundation model $\proc P$ down into an ``encoder'', a ``universal task'' and associated predictive ``model'', and a ``decoder'', such that the encoder, limited by an encoder-specific task representation, converts inputs to $\proc{P}$ into universal task model inputs, the predictive model attempts to solve the universal task over the modified inputs, and the decoder (similarly limited by a decoder-specific task representation) converts predictions for the universal task into predictions over the true labels $\Y$. 

To define this decomposition formally, we introduce the following variables, introduced in Notation~\ref{notation:fm_decomposition_elements}.

\begin{notation}
    \label{notation:fm_decomposition_elements}
    Let $\X, \T, \Y$ be the normal extended domain, with $\rv x, \rv t, \rv y$ the normal jointly distributed random variables. Then, define the following terminology and notational elements:
    \begin{description}[leftmargin=0pt]
        \item[Encoder and Decoder Task Representations] Let $\RR_E$ and $\RR_D$ be encoder and decoder specific task representation spaces, with $\rho_E:\kernelover{\T}{\RR_E}$ and $\rho_D: \kernelover{\T}{\RR_D}$ defining encoder and decoder representations, and $\rv \rE$ and $\rv \rD$ the induced encoder/decoder representation random variables, respectively. 
        \item[Global Representation] Define $\rho = \rho_E \otimes \rho_D$, with $\rv r$ the induced random variable.
        \item[Universal-task Domain] Let $\UX$, $\UY$ be alternative, standard-Borel spaces. Intuitively, $\UX$ and $\UY$ define the input and output spaces of the universal prediction problem to which the original task is reduced: the encoder maps the original input and task representation into $\UX$, and predictions in $\UY$ can subsequently be transformed by the decoder into predictions over the original label space $\Ybar$.
        \item[Encoder] Let $\proc E: \X \times \RR_E \to \UX$. In the context of a known encoder $\proc E$, we can introduce a random variable $\rv{\ux}$ defined as $\proc E(\rv x, \rv \rE)$, jointly distributed with all others.
        \item[Universal Task] Let $u$ define a probability distribution for a ``universal task'' predicting elements from $\UY$ given inputs from $\UX$, such that $u: \kernelover{\UX}{\UY}$. In the context of a defined task $u$ and encoder $\proc E$, we will let the random variable $\rv \uy$ denote the random variable induced from $u(\rv \ux)$, jointly distributed and independent of all others given $\rv \ux$.
        \item[Labeler] Define a per-task deterministic labeler $d: \RR_D \times \UY \to \Y$, a measurable function that maps a universal-task output $z \in \UY$ to a downstream label 
        for any represented task $\rD \in \RR_D$. We write $d_{\rD}(\cdot) = d(\rD, \cdot)$.
        \item[Model] Let $f: \UX \to \ProbDists{\UY}$. In the context of a known $\proc E$ we can define a random variable induced by $f$ as $\fapprox{\rv \uy}$ sampled via $f(\rv \ux)$, jointly distributed and independent with all others given $\rv \ux$.
        \item[Decoder] Let $\proc{D}: \ProbDists{\UY} \times \RR_D \to \ProbDists{\Ybar}$ be the realized decoder. It is intended to approximate the \emph{target decoder} $\Dstar: \ProbDists{\UY} \times \RR_D \to \ProbDists{\Ybar}$ induced by the labeler,
        \[
            \Dstar(q, \rD) := \pushforward{d_{\rD}}{q},
        \]
        i.e., the distribution of $d_{\rD}(z)$ for $z \sim q$. 
    \end{description}
\end{notation}

Unlike the model $f$ and decoder $\proc{D}$, which can be compared respectively to the target universal task $u$ and the stochastic labeling operation induced by $d$, we do not posit a target encoder. Rather, the choice of encoder $E$ is part of the specification of the universal-task reduction itself and therefore jointly determines, together with $u$, the universal prediction problem.

With these introduced, we can formally state and prove Theorem~\ref{thm:all_fms_reduce}, which informally states that all foundation models permit a reduction through this universal task lens.

\begin{theorem}[All Foundation Models are Universal Task Reductions]
    \label{thm:all_fms_reduce}

    Let $\proc{P}$ be a foundation model over an extended domain $\X, \T, \Ybar$. Then there exists a representation, encoder, model, and decoder (in the sense of the previously introduced terminology), such that almost surely
    
    \[\proc{P}(x, \rv t) = \proc{D}(f(\proc{E}(x, \rv \rE)), \rv \rD).\]
\end{theorem}
\begin{proof}
    We can realize this trivially by construction. For $t \in \T$, let the encoder representation $\rho_E(t) = \delta_t$ be complete, let $\UX = \X \times \T$, $\UY = \Ybar$, $\proc{E} = I$ be the identity, $f = \proc{P}$ directly, and let $\proc{D}(q, \cdot) = q$ return the first argument only. Then this equality holds almost surely.
\end{proof}

Note that Theorem~\ref{thm:all_fms_reduce} does not suggest that all foundation models \emph{uniquely} decompose among these components -- merely that there exists \emph{a} decomposition of this form.
Additionally, note that Theorem~\ref{thm:all_fms_reduce} leverages a consistent, deterministic representation $\rho$, and thus does not permit us to say that we can decompose a generic foundation model over $\T$ in a manner leveraging an inconsistent or stochastic representation. However, in such settings, if we assume the user is only permitted a single draw from the given representation $\rho$ (which we generally will), then we really are operating over the induced domain of $\X, \RR, \Ybar$ (rather than $\X, \T, \Ybar$) with a consistent, deterministic representation without loss of generality.
Finally, and most critically, note that the construction used in Theorem~\ref{thm:all_fms_reduce} is trivial, and does not assign any special properties to any components in a particular sense. However, it does show that we can, without loss of generality almost surely, use this decomposed view of a foundation model as a general definition. In particular, we can use Definition~\ref{def:reduced_fms} instead of Definition~\ref{def:foundation_model} as our working definition of ATI foundation models, and conclude by Theorem~\ref{thm:all_fms_reduce} that this is almost surely without loss of generality.

\begin{definition}[Foundation Models as Reductions]
    \label{def:reduced_fms}
    Define $\rho$, $\proc{E}$, $f$, $\proc{D}$  in the prior terminology block. Then a foundation model $\proc{P}: \X \times \T \to \ProbDists{\Y}$ can be defined as the tuple $(\rho, \proc{E}, f, \proc{D})$, realizing the function 
    \[\proc{P}(x, \rv t) = \proc{D}(f(\proc{E}(x, \rv \rE)), \rv \rD).\]
\end{definition}

\paragraph{Example}
A nontrivial instance of Theorem~\ref{thm:all_fms_reduce} and Definition \ref{def:reduced_fms} is given by
autoregressive foundation models for longitudinal EHR data
\citep{ETHOS}. Here, the model need not receive the downstream task as
input. Instead, we may take $\rho_E$ to be degenerate and let
$\proc E(x,\cdot)=x$, while choosing the universal task $u$ to be prediction
of a future patient trajectory $\rv z$ given the observed history $\rv x$.
The learned model $f$ approximates this trajectory distribution. A downstream
task $t$, such as predicting whether a particular clinical event occurs within
a specified time horizon, is represented on the decoder side through
$\rho_D$, and $\proc D$ maps the predicted trajectory distribution to a
distribution over the corresponding task label. Thus, many distinct
task-specific predictions are obtained by reducing them to the same universal
task of forecasting future patient trajectories.

\paragraph{Optimality Relative To a Target Universal Task $u$ And Labeler $d$}
Note that thus far, we have not explicitly used the universal task $u$ or the labeler $d$. To incorporate these into our definitions, we will define the notions of $u$ or $d$-optimal foundation models in Definition~\ref{def:consistent_fm}. Informally, a foundation model is $u$-optimal if $f$ perfectly reflects the distribution $u$ (e.g., $f(\ux) = u(\ux)$). 
It is $d$-optimal if the realized decoder coincides with the target decoder, $\proc D = \Dstar$, i.e., $\proc D(q, \rD) = \pushforward{d_{\rD}}{q}$.

\begin{definition}[Foundation Model Optimality]
    \label{def:consistent_fm}
    Let $(\rho, \proc E, f, \proc D)$ define a foundation model $\proc P$, with $u$, $d$ and the induced target decoder $\Dstar$ as defined previously. Then $\proc P$ is \emph{$u$-optimal} if $f = u$ (if $f$ perfectly models the universal task) and \emph{$d$-optimal} if $\proc D(q, \rD) = \Dstar(q, \rD)$ for all $(q, \rD)$ (if $\proc D$ exactly implements the labeler $d$).
\end{definition}

In general, $u$ and $d$-optimality are only meaningful for a pre-specified $u$ or $d$; e.g., the construction used in Theorem~\ref{thm:all_fms_reduce} actually establishes $u$-optimality for $u$ being the kernel for $\proc{P}(\rv x, \rv t)$ and $d$-optimality with $d_{\rD}(z) = z$ (the identity, since $\UY = \Ybar$), thus showing that one can always choose a $u$ and $d$ to establish optimality. However, this will not imply that for any $u$ and $d$, a given foundation model is optimal.

\paragraph{Notational Elements for the Decomposed Foundation Model}

In addition to the variables introduced in Notation~\ref{notation:fm_decomposition_elements}, now that we have introduced $u$ and $d$ optimality, it makes sense to introduce explicit random variables for the various combinations of the components of $\proc P$ under Definition~\ref{def:reduced_fms}, as they will ultimately help us describe $\proc P$ in more precise ways. These elements are defined in Notation~\ref{notation:fm_augmented_rvs} and are represented as an augmented graphical model in Figure~\ref{fig:graphical_model_fm_error}. 
Tables \ref{tab:universal_task_spaces} and \ref{tab:universal_task_reduction} summarize all the introduced objects and how they relate to each other.

\begin{notation}
    \label{notation:fm_augmented_rvs}
    Recall the notation introduced in Notation~\ref{notation:fm_decomposition_elements}. We further introduce the following notational conventions: (1) The use of a $\fapprox{\cdot}$ will indicate a random variable that has been approximated via $f$ rather than the true universal task $u$ -- this is consistent with $\fapprox{\rv \uy}$ already introduced in Notation~\ref{notation:fm_decomposition_elements}. (2) The use of a $\Dapprox{\cdot}$ will indicate a distribution being approximated through the action of the realized decoder $\proc D$ rather than the target decoder $\Dstar$.

    Consistent with these, we have:
    \begin{description}[leftmargin=0pt]
        \item[Input Variables] $\rv x, \rv t, \rv y, \rv \rE, \rv \rD$ as defined in Notation~\ref{notation:fm_decomposition_elements}.
        \item[Encoder Action] Let $\rv \ux = \proc E(\rv x, \rv \rE)$ be the random variable induced by $\proc E$.
        \item[Universal Task Outputs] Let $\rv \uy$ be the true universal task label drawn from $u(\rv \ux)$, and $\fapprox{\rv \uy}$ be the approximated universal task label drawn from $f(\rv \ux)$. A $u$-optimal model has $\gv{\rv \uy}{\rv \ux}$ and $\gv{\fapprox{\rv \uy}}{\rv \ux}$ identically distributed.
        \item[Decoded Outputs] We will let the ``\duy'' character represent the $\Ybar$ valued random variable obtained by applying the target decoder $\Dstar$ or the realized decoder $\proc D$ to a distribution over $\UY$. Concretely, 
        \begin{itemize}
            \item Let $\rv \duy$ represent the random variable obtained by applying the target decoder to the true universal-task distribution:
            \[
                \rv \duy \sim \Dstar\!\left(u(\rv \ux), \rv \rD\right).
            \]
            \item Let $\Dapprox{\rv \duy}$ represent the random variable obtained by applying the realized decoder to the true universal-task distribution:
            \[
                \Dapprox{\rv \duy} \sim \proc D\!\left(u(\rv \ux), \rv \rD\right).
            \]
            \item Let $\fapprox{\rv \duy}$ represent the random variable obtained by applying the target decoder to the model-approximated universal-task distribution:
            \[
                \fapprox{\rv \duy} \sim \Dstar\!\left(f(\rv \ux), \rv \rD\right).
            \]

            \item Instead of using the awkward notation $\fapprox{\Dapprox{\rv \duy}}$ for our final term, we use $\fapprox{\rv y}$ to represent the true output of $\proc P$,  obtained by applying the realized decoder to the realized model output:
            \[
                \fapprox{\rv y} \sim \proc D\!\left(f(\rv \ux), \rv \rD\right).
            \]
        \end{itemize}
        Note that a $d$-optimal model has $\gv{\rv \duy}{\rv x, \rv r}$ and $\gv{\Dapprox{\rv \duy}}{\rv x, \rv r}$ identically distributed, and likewise $\gv{\fapprox{\rv \duy}}{\rv x, \rv r}$ and $\gv{\fapprox{\rv y}}{\rv x, \rv r}$. If $f$ is $u$-optimal and $D$ is $d$-optimal, then
            \[
            \gv{\rv \duy}{\rv x, \rv r}
            =
            \gv{\Dapprox{\rv \duy}}{\rv x, \rv r}
            =
            \gv{\fapprox{\rv \duy}}{\rv x, \rv r}
            =
            \gv{\fapprox{\rv y}}{\rv x, \rv r}
            \quad \text{a.s.}
            \]
    \end{description}
    The conditional independence relationships between these random variables, and their generating kernels/relationships, are shown via the augmented graphical model in Figure~\ref{fig:graphical_model_fm_error}.
\end{notation}


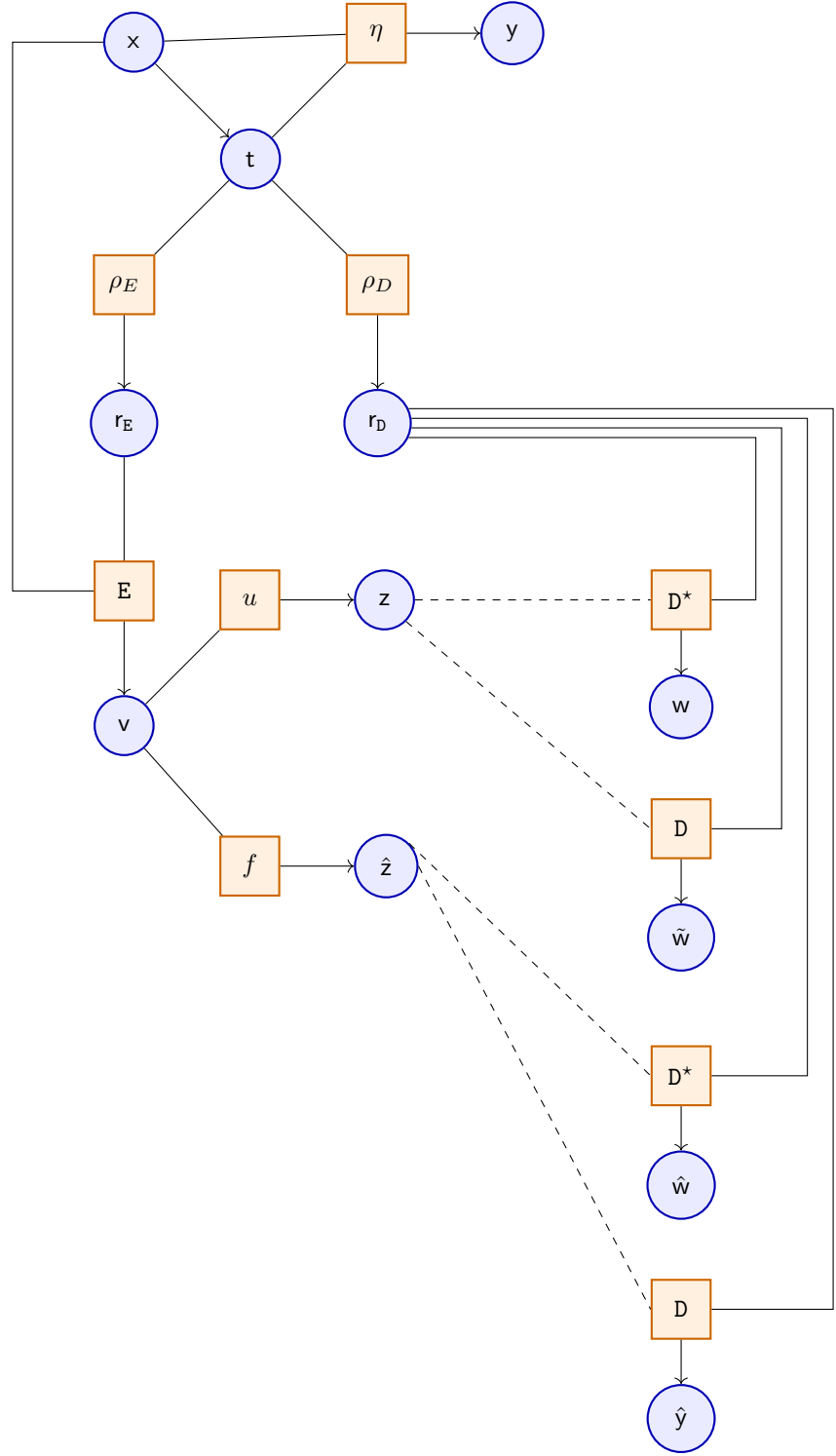
\begin{figure}
    \centering
    \begin{tikzpicture}[
        rv/.append style={
            draw=blue!70!black,
            fill=blue!8,
            thick
        },
        kernel/.append style={
            draw=orange!80!black,
            fill=orange!12,
            thick
        }
    ]


        \node[style=rv] (x)
            {$\rv x$};

        \node[style=rv] (t)
            [below right=of x]
            {$\rv t$};

        \node[style=kernel] (eta)
            [above right=of t]
            {$\eta$};

        \node[style=rv] (y)
            [right=of eta]
            {$\rv y$};

        \draw[->] (x) -- (t);
        \draw[-]  (t) -- (eta);
        \draw[-]  (x) -- (eta);
        \draw[->] (eta) -- (y);


        \node[style=kernel] (rhoE)
            [below left=of t]
            {$\rho_E$};

        \node[style=rv] (rE)
            [below=of rhoE]
            {$\rv \rE$};

        \node[style=kernel] (rhoD)
            [below right=of t]
            {$\rho_D$};

        \node[style=rv] (rD)
            [below=of rhoD]
            {$\rv \rD$};

        \draw[-]  (t) -- (rhoE);
        \draw[->] (rhoE) -- (rE);

        \draw[-]  (t) -- (rhoD);
        \draw[->] (rhoD) -- (rD);


        \node[style=kernel] (E)
            [below=14mm of rE]
            {$\proc E$};

        \node[style=rv] (ux)
            [below=of E]
            {$\rv \ux$};

        \draw[-]
            (rE.south) -- (E.north);

        \coordinate (xEleft)
            at ($(E.west)+(-11mm,0)$);

        \coordinate (xEtop)
            at (xEleft |- x.west);

        \draw[-]
            (x.west)
            -- (xEtop)
            -- (xEleft)
            -- (E.west);

        \draw[->]
            (E.south) -- (ux.north);


        \node[style=kernel] (u)
            [above right=of ux]
            {$u$};

        \node[style=rv] (uy)
            [right=of u]
            {$\rv \uy$};

        \node[style=kernel] (f)
            [below right=12mm and 10mm of ux]
            {$f$};

        \node[style=rv] (fuy)
            [right=of f]
            {$\fapprox{\rv \uy}$};

        \draw[-]  (ux) -- (u);
        \draw[-]  (ux) -- (f);

        \draw[->] (u) -- (uy);
        \draw[->] (f) -- (fuy);


        \node[style=kernel] (du)
            [right=32mm of uy]
            {$\Dstar$};

        \node[style=rv] (duy)
            [below=6mm of du]
            {$\rv \duy$};

        \node[style=kernel] (Du)
            [below=8mm of duy]
            {$\proc D$};

        \node[style=rv] (Duy)
            [below=6mm of Du]
            {$\Dapprox{\rv \duy}$};

        \node[style=kernel] (df)
            [below=10mm of Duy]
            {$\Dstar$};

        \node[style=rv] (dfuy)
            [below=6mm of df]
            {$\fapprox{\rv \duy}$};

        \node[style=kernel] (Df)
            [below=8mm of dfuy]
            {$\proc D$};

        \node[style=rv] (Dfuy)
            [below=6mm of Df]
            {$\fapprox{\rv y}$};

        \draw[->] (du.south) -- (duy.north);
        \draw[->] (Du.south) -- (Duy.north);
        \draw[->] (df.south) -- (dfuy.north);
        \draw[->] (Df.south) -- (Dfuy.north);

        %

        \draw[dashed]
            (uy.east) -- (du.west);

        \draw[dashed]
            (uy.south east) -- (Du.west);

        \draw[dashed]
            (fuy.north east) -- (df.west);

        \draw[dashed]
            (fuy.east) -- (Df.west);

        %

        \coordinate (laneBase)
            at ($(du.east)+(6mm,0)$);

        \foreach \target/\lane/\ang in {du/0/-25, Du/1/-8, df/2/8, Df/3/25} {
            \coordinate (lane)
                at ($(laneBase)+(\lane*3.5mm,0)$);
            \draw[-]
                (rD.\ang)
                -| (lane |- \target.east)
                -- (\target.east);
        }

    \end{tikzpicture}

    \caption{
        An augmented schematic graphical representation of the random
        variables and distribution-level operations induced by different
        combinations of our foundation model components.
        Blue circular nodes correspond to random variables and orange
        nodes correspond to stochastic or deterministic kernels or
        procedures.
        Solid directed edges through kernels indicate sample-level
        generation.
        Dashed edges from universal-task output random variables to
        $\Dstar$ or $\proc D$ indicate dependence on the corresponding
        conditional distribution rather than pointwise application to the
        realized random variable.
        Each decoder (target $\Dstar$ or realized $\proc D$) separately
        receives the decoder-side task representation $\rv \rD$.
        This figure should be interpreted as a probabilistic graphical
        representation rather than a causal DAG.
    }

    \label{fig:graphical_model_fm_error}
\end{figure}


\definecolor{targetrow}{RGB}{242,247,255}
\definecolor{realizedrow}{RGB}{255,247,238}



\definecolor{targetrow}{RGB}{242,247,255}
\definecolor{realizedrow}{RGB}{255,247,238}

\begin{table*}[!t]
    \centering
    \footnotesize
    \setlength{\tabcolsep}{5pt}
    \renewcommand{\arraystretch}{1.15}

    \begin{tabularx}{\textwidth}{
        @{}
        >{\raggedright\arraybackslash}p{0.19\textwidth}
        >{\raggedright\arraybackslash}p{0.30\textwidth}
        >{\raggedright\arraybackslash}X
        @{}
    }

        \toprule

        \textbf{Space}
        &
        \textbf{Elements / random variable}
        &
        \textbf{Interpretation}
        \\

        \midrule

        Input space $\mathcal X$
        &
        $x\sim\rv x,\quad x\in\mathcal X$
        &
        Original data/input space.
        \\

        Task space $\mathcal T$
        &
        $t\sim\rv t,\quad t\in\mathcal T$
        &
        Space of downstream tasks of interest.
        \\

        Label space $\Ybar$ 
        &
        $y\sim\rv y,\quad y\in\Ybar$
        &
        Downstream label space.
        \\

        Encoder representation space $\mathcal R_E$
        &
        $r_E\sim\rv\rE,\quad r_E\in\mathcal R_E$
        &
        Space in which encoder-side representations of the task live.
        \\

        Decoder representation space $\mathcal R_D$
        &
        $r_D\sim\rv\rD,\quad r_D\in\mathcal R_D$
        &
        Space in which decoder-side representations of the task live.
        \\

        Global representation space
        $\mathcal R_E\times\mathcal R_D$
        &
        $r=(r_E,r_D)\sim\rv r$
        &
        Joint task-representation space combining the information used
        by the encoder and decoder.
        \\

        Universal-task input space $\UX$
        &
        $v\sim\rv\ux,\quad v\in\UX$
        &
        Input space of the universal task.
        \\

        Universal-task output space $\UY$
        &
        $z\sim\rv\uy,\quad z\in\UY$
        &
        Output space of the universal task.
        \\

        \bottomrule

    \end{tabularx}

    \caption{
        Spaces used in the universal-task reduction and their associated
        realizations and random variables.
    }

    \label{tab:universal_task_spaces}
\end{table*}

\begin{table*}[!t]
    \centering
    \footnotesize
    \setlength{\tabcolsep}{3.2pt}
    \renewcommand{\arraystretch}{1.18}

    \begin{tabularx}{\textwidth}{
        @{}
        >{\raggedright\arraybackslash}p{0.075\textwidth}
        >{\raggedright\arraybackslash}p{0.145\textwidth}
        >{\raggedright\arraybackslash}p{0.19\textwidth}
        >{\raggedright\arraybackslash}p{0.245\textwidth}
        >{\raggedright\arraybackslash}X
        @{}
    }

        \toprule

        \textbf{Role}
        &
        \textbf{Object}
        &
        \textbf{Definition / type}
        &
        \textbf{Induced object(s)}
        &
        \textbf{Interpretation}
        \\

        \midrule


        \multicolumn{5}{@{}l}{\textbf{Base problem and reduction}}\\
        \addlinespace[1pt]

        Base
        &
        Task kernel $\eta$
        &
        $\eta:
        \mathcal X\times\mathcal T
        \rightsquigarrow\Ybar$
        &
        $\rv y
        \mid
        \rv x=x,\rv t=t
        \sim
        \eta(\cdot\mid x,t)$
        &
        True task-specific conditional distribution of labels.
        \\

        Reduction
        &
        Encoder task representation $\rho_E$
        &
        $\rho_E:
        \mathcal T
        \rightsquigarrow
        \mathcal R_E$
        &
        $\rv\rE
        \mid \rv t=t
        \sim
        \rho_E(t)$
        &
        Represents the task for the purpose of constructing the
        universal-task input.
        \\

        Reduction
        &
        Decoder task representation $\rho_D$
        &
        $\rho_D:
        \mathcal T
        \rightsquigarrow
        \mathcal R_D$
        &
        $\rv\rD
        \mid \rv t=t
        \sim
        \rho_D(t)$
        &
        Represents the task for the purpose of interpreting
        universal-task outputs in the original label space.
        \\

        Reduction
        &
        Global representation $\rho$
        &
        $\rho
        =
        \rho_E\otimes\rho_D:
        \mathcal T
        \rightsquigarrow
        \mathcal R_E\times\mathcal R_D$
        &
        $\rv r=(\rv\rE,\rv\rD)$
        &
        Joint representation containing the encoder- and decoder-side
        task representations.
        \\

        Reduction
        &
        Encoder $\proc E$
        &
        $\proc E:
        \mathcal X\times\mathcal R_E
        \to
        \UX$
        &
        $\rv\ux
        =
        \proc E(\rv x,\rv\rE)$
        &
        Maps the original input and encoder-side task representation
        into the input space of the universal task.
        \\

        \addlinespace[2pt]
        \midrule


        \multicolumn{5}{@{}l}{\textbf{Universal task}}\\
        \addlinespace[1pt]

        \rowcolor{targetrow}
        Target
        &
        Universal task $u$
        &
        $u:
        \UX
        \rightsquigarrow
        \UY$
        &
        $\rv\uy
        \mid\rv\ux=v
        \sim
        u(v)$
        &
        Target conditional distribution defining the universal task.
        \\

        \rowcolor{realizedrow}
        Realized
        &
        Model $f$
        &
        $f:
        \UX
        \to
        \mathcal P(\UY)$
        &
        $\fapprox{\rv\uy}
        \mid\rv\ux=v
        \sim
        f(v)$
        &
        Realized predictive model intended to approximate the
        universal task $u$.
        \\

        \addlinespace[2pt]
        \midrule


        \multicolumn{5}{@{}l}{\textbf{Task-specific labeling and decoding}}\\
        \addlinespace[1pt]


        \rowcolor{targetrow}
        Target
        &
        Labeler $d$
        &
        $d:
        \mathcal R_D\times\UY
        \to
        \Ybar$
        &
        \begin{tabular}[t]{@{}l@{}}
            $d_{\rD}(z)=d(\rD,z)$\\
        \end{tabular}
        &
        Task-specific deterministic map from a universal-task outcome
        to a downstream label. 
        \\

        \rowcolor{targetrow}
        Target
        &
        Target decoder $\Dstar$
        &
        \begin{tabular}[t]{@{}l@{}}
            $\Dstar:
            \mathcal P(\UY)\times\mathcal R_D
            \to
            \mathcal P(\Ybar)$,\\
            $\Dstar(q,\rD)=\pushforward{d_{\rD}}{q}$
        \end{tabular}
        &
        \begin{tabular}[t]{@{}l@{}}
            $\rv\duy
            \sim
            \Dstar(u(\rv\ux),\rv\rD)$,\\
            $\fapprox{\rv\duy}
            \sim
            \Dstar(f(\rv\ux),\rv\rD)$
        \end{tabular}
        &
        Decoding operation implied by the target labeler.
        \\

        \rowcolor{realizedrow}
        Realized
        &
        Decoder $\proc D$
        &
        $\proc D:
        \mathcal P(\UY)\times\mathcal R_D
        \to
        \mathcal P(\Ybar)$
        &
        \begin{tabular}[t]{@{}l@{}}
            $\Dapprox{\rv\duy}
            \sim
            \proc D(u(\rv\ux),\rv\rD)$,\\
            $\fapprox{\rv y}
            \sim
            \proc D(f(\rv\ux),\rv\rD)$
        \end{tabular}
        &
        Realized decoder intended to approximate the target decoder
        $\Dstar$.
        \\

        \addlinespace[2pt]
        \midrule


        \multicolumn{5}{@{}l}{\textbf{Complete foundation model}}\\
        \addlinespace[1pt]

        \rowcolor{realizedrow}
        Realized
        &
        Foundation model $\proc P$
        &
        $\proc P:
        \mathcal X\times\mathcal T
        \to
        \mathcal P(\Ybar)$
        &
        \begin{tabular}[t]{@{}l@{}}
            $\fapprox{\rv y}
            \sim
            \proc P(\rv x,\rv t)$,\\[2pt]
            $\proc P(x,t)
            =
            \proc D\!\left(
                f\!\left(
                    \proc E(x,r_E)
                \right),
                r_D
            \right)$
        \end{tabular}
        &
        Complete any-task inference procedure obtained by composing
        the task representations, encoder, realized universal-task
        model, and decoder.
        \\

        \bottomrule

    \end{tabularx}

    \caption{
        Components of the universal-task reduction. Reduction components
        specify how the original task and input are mapped into the
        universal-task domain. Target components, shown in blue, specify
        the intended universal task, labeler, and the target decoder they induce; realized components, shown in orange, denote
        their implemented counterparts.
    }

    \label{tab:universal_task_reduction}
\end{table*}

\section{Decomposing Foundation Model Error}
\label{sec:FM_error_sources}
In this section, we will use Definition~\ref{def:reduced_fms} to decompose foundation model errors into terms corresponding to different aspects of this decomposition. There are many such decompositions possible, and as we will define them here, many are near-trivial in complexity, as we will define individual error terms such that by construction they sum to the total foundation model error. Nevertheless, the decompositions we introduce will allow us to understand what parts of the error we can optimize and control under different algorithms. For example, suppose a foundation model is trained in a ``typical'' manner, where the developer's intent is to train the model $f$ to approximate a pre-specified universal task $u$ and the decoder $\proc D$ to approximate the target decoder $\Dstar$, while the other components of the process are held fixed. A typical training procedure in this setting would then optimize $f$ to better approximate $u$ and/or $\proc D$ to better approximate $\Dstar$. By separating such error components out from other aspects of the error, we can identify what remaining components exist in the total error and when they can be bounded or identified to be zero.

To decompose error in this way, we will first formally define what we mean by ``foundation model error'' in Subsection~\ref{subsec:fm_error_def}, then introduce a large set of terms quantifying different aspects of the total error in Subsection~\ref{subsec:fm_error_components}, before finally formally introducing our error decompositions and arguing for their utility in Subsection~\ref{subsec:fm_error_decomp}.

\subsection{Defining Error for a Single-Draw Model}
\label{subsec:fm_error_def}
We define the error of a foundation model under strictly proper scoring loss $\ell$ formally in Definition~\ref{def:fm_error}.

\begin{definition}[Foundation Model $\ell$-Error]
\label{def:fm_error}
    Let $\X, \T, \Ybar$ be an extended domain, $\rho: \kernelover{\T}{\RR}$ be a task representation, $\rv x, \rv t, \rv r, \rv y$ be the normal joint random variables, and let $\ell$ be a strictly proper scoring loss.
    Let $\proc{P}$ be a single-$\rho$-draw foundation model over this domain, defined therefore via Definition~\ref{def:foundation_model} as a procedure over $\X \times \RR$. Let $\hat{\rv y}$ be a random variable induced by $\proc{P}(\rv x, \rv r)$, which we can therefore understand to be jointly distributed with all other random variables.
    We define the foundation model $\proc{P}$'s $\ell$ error over $\X, \T, \Ybar$ as the normal regret divergence for $\ell$ between the predicted distribution given by $\proc P$ and the true distribution given by $\proc P^* = \eta$ in expectation over samples and tasks:
    \begin{align*}
        \Err{\ell}{\proc{P}}
            &:= d_\ell(\eta(\rv x, \rv t), \proc{P}(\rv x, \rv r | \rv t)) \\
            &= \EV{\rv x, \rv t}{
                \EV{\rv r | \rv t = t}{
                    d_\ell(\eta(x, t), \proc P(x, r))
                }
            } \\ 
    \end{align*}
\end{definition}

Definition~\ref{def:fm_error} introduces the error of the $\X \times \RR$ procedure against the ideal model achievable over $\X \times \T$ (to respect the single-draw constraint). To capture the impact the single-draw criteria has, Lemma~\ref{lemma:fm_error_reorg} decomposes $\Err{\ell}{\proc P}$ into a term describing the error in comparison to the achievable error of $\proc P$ vs. the ideal predictor over $\X \times \RR$ (not $\X \times \T$) and a term capturing remaining error due to task-representation error.

\begin{lemma}
    \label{lemma:fm_error_reorg}
    Given Definition~\ref{def:fm_error} and $\phi_\rho: \kernelover{\X \times \RR}{\Ybar}$ introduced in Definition~\ref{def:induced_kernel_phi}, we can re-organize $\Err{\ell}{\proc P}$ as defined in Definition~\ref{def:fm_error} as follows:
    \begin{align*}
        \Err{\ell}{\proc{P}}
            &= \underbrace{\divergence{\ell}{\phi_\rho}{ \proc P}}_{\text{achievable error}} + \representationErr \\
            &= \underbrace{\condDivergence{\ell}{\rv y}{\fapprox{\rv y}}{\rv x, \rv r}}_{\text{achievable error}} + \representationErr
    \end{align*}
\end{lemma}
\begin{proof}
    \begin{align*}
        \Err{\ell}{\proc{P}}
            &= \EV{\rv x, \rv t}{
                \EV{\rv r | \rv t = t}{
                    d_\ell(\eta(x, t), \proc P(x, r))
                }
            } \\ 
            &= \EV{\rv x, \rv t}{
                \EV{\rv r | \rv t = t}{
                    \EV{\rv y | \rv x = x, \rv t = t, \rv r  = r}{
                        \ell(\proc P(x, r), y)
                        -
                        \ell(\eta(x, t), y)
                    }
                }
            } \\
            &= \EV{\rv x, \rv t}{
                \EV{\rv r | \rv t = t}{
                    \EV{\rv y | \rv x = x, \rv t = t, \rv r = r}{
                        \ell(\proc P(x, r), y)
                    }
                }
                -
                \EV{\rv y | \rv x = x, \rv t = t}{
                    \ell(\eta(x, t), y)
                }
            } \\
            &= \EV{\rv x, \rv t}{
                \EV{\rv r | \rv t = t, \rv x = x}{
                    \EV{\rv y | \rv x = x, \rv r = r}{
                        \ell(\proc P(x, r), y)
                    }
                }
            }
            -
            \EV{\rv x, \rv t}{
                \EV{\rv y | \rv x = x, \rv t = t}{
                    \ell(\eta(x, t), y)
                }
            } \\
            &= \EV{\rv t}{
                \EV{\rv r, \rv x | \rv t = t}{
                    d_\ell(\phi_\rho(x, r), \proc P(x, r)) + h_\ell(\phi_\rho(x, r))
                }
            }
            -
            h_\ell(\rv y | \rv x, \rv t) \\
            &= \EV{\rv r, \rv x}{
                \EV{\rv t = t | \rv r=r, \rv x=x}{
                    d_\ell(\phi_\rho(x, r), \proc P(x, r)) + h_\ell(\phi_\rho(x, r))
                }
            }
            -
            h_\ell(\rv y | \rv x, \rv t) \\
            &= \divergence{\ell}{\phi_\rho(\rv x, \rv r)}{\proc P(\rv x, \rv r)} + h_\ell(\rv y | \rv x, \rv r) - h_\ell(\rv y | \rv x, \rv t) \\
            &= \divergence{\ell}{\phi_\rho}{\proc P} + \representationErr,
    \end{align*}
    as desired.
\end{proof}


\subsection{Components of Foundation Model Error}
\label{subsec:fm_error_components}
To provide a useful decomposition of $\Err{\ell}{\proc P}$, in this section we will introduce a variety of terms of different types of error which we will eventually combine in Theorem~\ref{thm:error_decomposition}.

These terms will be organized to quantify the errors of different components of $\proc P$ in various arrangements, motivated by the assumption that the developer is ``intending'' to approximate a pre-specified universal task $u$ and target decoder $D^\star$. For example, under the assumption that a model builder may be training the model $f$ to approximate a task $u$, it would be of interest to compute the error of $f$'s approximation of $u$ (to determine how well their training process worked) or how much error they would incur were they able to approximate $u$ perfectly, then use the rest of their foundation model as specified (to evaluate the choice of universal task $u$).

We will introduce these terms in three sections. First, a section introducing ``error'' terms capturing unambiguous, strictly non-negative measures of error that would be obtained were elements of $\proc P$ isolated from the system and considered in a standalone manner. Second, a section introducing signed ``excess'' terms that capture how real-world interactions amongst these components in a full foundation model can either exacerbate or ameliorate total expected error. Finally, third, we will note that as $\ell$ is a strictly proper scoring loss, there are natural, information-geometric, signed Bregman interaction terms canonical via the Bregman three-point identity between some combinations of our terms. 
All of these sections will rely heavily on the notation introduced in Appendix~\ref{sec:notation} and the specific random variables introduced to correspond to foundation model components defined in Notation~\ref{notation:fm_decomposition_elements} and \ref{notation:fm_augmented_rvs} and Figure~\ref{fig:graphical_model_fm_error}. In addition, all of these terms will judge the foundation model (for now) against the achievable total error, and thus condition on the observed $(\rv x, \rv r)$ pair rather than the base $(\rv x, \rv t)$ pair, consistent with the single-draw criteria.

\paragraph{Error Terms}
In Definition~\ref{def:error_terms}, we introduce various direct error measurements corresponding to parts of $\proc P$. While we will combine each of these terms later, informally, they are all directly useful quantities when designing or evaluating a concrete foundation model. Specifically:
\begin{itemize}
    \item A target universal task $u$ that has large deficiency $\uDef$ is likely a poor target for foundation model development, as solving the task perfectly will still not enable production of a consistent, high-performance foundation model.
    \item While often unmeasurable, the observed model deficiency $\fDef$ directly captures how much error is inescapable with a fixed model $f$, and will be useful in our decomposition later.
    \item If the model error $\fModelErr$ is large, then the model poorly approximates $u$, limiting internal validity if the choice of universal task $u$ is important.
    \item If labeler error w.r.t. either $u$ or $f$ is large, then the choice of labeler $d$ can be seen to be problematic at a design level.
    \item Finally, if the observed decoder error w.r.t. either $u$ or $f$ is large, it means that the realized decoder $\proc D$ is unable to well approximate the target decoder $\Dstar$, limiting internal validity if the choice of $d$ is important.
\end{itemize}

All terms will also be w.r.t. a strictly proper $\Ybar$-scoring loss
$\ell$. In addition, we let
$\ell_u:\ProbDists{\UY}\times\UY\to\overline{\R}$
denote a strictly proper scoring loss over the universal-task output
space $\UY$, which we use to measure discrepancies between the target
universal task $u$ and the model $f$.

\begin{definition}
    \label{def:error_terms}
    Given the setup introduced previously, define the following error terms, which are all non-negative:
    \begin{description}
        \item[$u$-deficiency $\uDef$] The \emph{universal-task deficiency} is the minimal possible error a $u$-optimal foundation model $\proc P$ could ever achieve when its decoder is restricted to a labeling step:\footnote{Meaning that $\proc P$ must be $d$-optimal for some labeling kernel $\lambda \in \MarkovKernels{\RR_D \times \UY}{\Ybar}$. We allow $\lambda$ to be stochastic so that $\uDef$ is the best error attainable from $\rv \uy$ at all; the cost of restricting to a deterministic labeler $d$ is then part of $\dErrU$.}
        \[\uDef := \condDeficiency{\ell}{\rv y}{\rv \uy, \rv \rD}{\rv x, \rv r}.\]
        \item[$f$-deficiency $\fDef$] The \emph{model-deficiency} is analogous to $\uDef$, but when we use the model $f$ rather than the target universal task $u$:
        \[\fDef := \condDeficiency{\ell}{\rv y}{\fapprox{\rv \uy}, \rv \rD}{\rv x, \rv r}.\]
        \item[Model Error $\fModelErr$] The \emph{model-error} captures how effectively $f$ models $u$, in the typical way:
        \[\fModelErr := \condDivergence{\ell_u}{\rv{\uy}}{\fapprox{\rv{\uy}}}{\rv x, \rv r}.\]
        \item[Labeler Error $\dErrU$] The \emph{labeler-error w.r.t. $u$} computes how much excess error is induced through the use of a specific labeler $d$ over the minimal error possible to obtain when predicting from the true universal task $u$:
        \[\dErrU := \condDivergence{\ell}{\rv y}{\rv \duy}{\rv x, \rv r} - \uDef.\]
        \item[Labeler Error $\dErrF$] Similarly, the \emph{labeler error w.r.t. $f$} computes the excess error due to $d$ when predicting from real model $f$:
        \[\dErrF := \condDivergence{\ell}{\rv y}{\fapprox{\rv \duy}}{\rv x, \rv r} - \fDef.\]
        \item[Decoder-labeler Error w.r.t $u$ $\DErrU$] The \emph{Decoder-error w.r.t. $u$} computes how well the realized decoder $\proc D$ approximates the target decoder $\Dstar$ on the outputs of the universal task $u$:
        \[\DErrU := \condDivergence{\ell}{\rv \duy}{\Dapprox{\rv \duy}}{\rv x, \rv r}\]
        \item[Decoder-labeler Error w.r.t $f$ $\DErrF$] The \emph{Decoder-error w.r.t $f$} is analogous to $\DErrU$ but over the outputs of realized model $f$ rather than $u$\footnote{Recall that $\fapprox{\rv y}$ denotes the realized decoder applied to the realized model, i.e., the output of $\proc P$.}:
        \[\DErrF := \condDivergence{\ell}{\fapprox{\rv \duy}}{\fapprox{\rv y}}{\rv x, \rv r}.\]
    \end{description}
\end{definition}

With these error terms, we can additionally introduce Definition~\ref{def:well-founded}, which states that a foundation model is \emph{well-founded} w.r.t. either $u$ or $d$ if it has no universal-task deficiency or labeler error, respectively. Informally, this definition is advisable as if one can justify that a foundation model setup is well-founded, it indicates they have chosen a universal task and labeler that theoretically permit any-task inference at optimal achievable performance, in the event that one can learn a $u$-optimal model $f$ and leverage a $d$-optimal decoder $\proc D$.
\begin{definition}
    \label{def:well-founded}
    A foundation model $\proc P$ is \emph{well-founded w.r.t. $u$} if $\uDef = 0$. It is \emph{well-founded w.r.t. $d$} if $\dErrU = 0$. It is simply \emph{well-founded} if it is well-founded w.r.t. both $u$ and $d$.
\end{definition}


In addition, Lemma~\ref{lemma:u-d-optimal-err} relates $u$- and $d$-optimality to the corresponding error terms, showing that $u$-optimality is characterized by zero model error, while $d$-optimality guarantees zero decoder error.

\begin{lemma}
    \label{lemma:u-d-optimal-err}
    Recall Definitions~\ref{def:consistent_fm} and \ref{def:error_terms}. Two properties hold:
    \begin{enumerate}
        \item $\proc P$ is $u$-optimal iff $\fModelErr = 0$.
        \item If $\proc P$ is $d$-optimal, then $\DErrF = \DErrU = 0$.
    \end{enumerate}
\end{lemma}
\begin{proof}
    The first property follows directly from the definition of $u$-optimality meaning that $f = u$ and that $\ell_u$ is strictly proper. The second follows from the definition of $d$-optimality meaning that $\proc D = \Dstar$, so $\Dapprox{\rv \duy}$ and $\rv \duy$ (resp. $\fapprox{\rv y}$ and $\fapprox{\rv \duy}$) have the same conditional distribution given $(\rv x, \rv r)$.
\end{proof}

\paragraph{Excess Terms}
All of the direct error terms in Definition~\ref{def:error_terms} are non-negative, and in a strict sense only a subset of them will directly contribute to total foundation model error. This discrepancy is because these error terms measure direct comparisons on individual components of $\proc P$, and when these components are combined, the result can have greater or lesser error than that introduced by the sum of its parts. To capture these additional interactions, Definition~\ref{def:excess_terms} defines these signed ``excess terms'' which compute what remains between the true foundation model error and the direct error terms defined in Definition~\ref{def:error_terms}. 

\begin{definition}
    \label{def:excess_terms}
    Given the setup introduced previously and the terms in Definition~\ref{def:error_terms}, define the following ``excess terms'', which compare the relative errors that one obtains in reality (via model $f$) vs. those they would expect during model design (via universal task $u$):
    \begin{description}
        \item[Excess model deficiency $\excessFDeficiency$] Even though $f$ is an imperfect approximator of $u$, it may exacerbate or ameliorate the total possible error obtainable through a permissible foundation model in comparison to $u$, and $\excessFDeficiency$ captures this tradeoff:
        \[\excessFDeficiency := \fDef - \uDef.\]
        \item[Excess labeler error $\excessdErr$] Similarly, beyond the optimal possible discrepancy in error, $f$ may induce different observed error relative to the possible error in comparison to what would be expected at design time w.r.t. the universal task $u$:
        \[\excessdErr := \dErrF - \dErrU.\]
        \item[Excess Decoder-labeler error $\excessDErr$] Similarly, the realized decoder can exacerbate or ameliorate error in matching the target decoder $\Dstar$ differently under $f$ in comparison to what would be expected at design time w.r.t $u$:
        \[\excessDErr := \DErrF - \DErrU\]
    \end{description}
\end{definition}

Building on these definitions, Lemma~\ref{lemma:excess_optimal} remarks that all excess terms $\Phi_*$ vanish under $u$-optimality and $\excessDErr$ vanishes under $d$-optimality.

\begin{lemma}
    \label{lemma:excess_optimal}
    If $\proc P$ is $u$-optimal, then $\excessFDeficiency = \excessdErr = \excessDErr = 0$. If $\proc P$ is $d$-optimal, then $\excessDErr = 0$.
\end{lemma}
\begin{proof}
    If $\proc P$ is $u$-optimal, then $f$ is precisely $u$, so by construction these differences in response between $f$ and $u$ vanish. For point 2, Lemma~\ref{lemma:u-d-optimal-err} shows that under $d$-optimality, both $\DErrF = \DErrU = 0$, which implies that $\excessDErr = 0$, as desired.
\end{proof}

\paragraph{Bregman Interaction Terms}
Lastly, in Definition~\ref{def:bregman_interactions} we introduce some notation for the remaining signed Bregman interaction terms inherent from Bregman's three-point identity and the fact that $\ell$ is a strictly proper scoring loss. 

\begin{definition}
    \label{def:bregman_interactions}
    We will explicitly use $\BregmanF$ to denote the specific Bregman interaction term capturing the difference in divergence between $\rv y$ and $\fapprox{\rv y}$ vs. that captured consecutively between $\rv y$ then $\fapprox{\rv \duy}$ then $\fapprox{\rv y}$:
    
    \begin{align*}
        \BregmanF &:= \condBregman{\ell}{\rv y}{\fapprox{\rv \duy}}{\fapprox{\rv y}}{\rv x, \rv r} := \condDivergence{\ell}{\rv y}{\fapprox{\rv y}}{\rv x, \rv r} - \condDivergence{\ell}{\rv y}{\fapprox{\rv \duy}}{\rv x, \rv r} - \condDivergence{\ell}{\fapprox{\rv \duy}}{\fapprox{\rv y}}{\rv x, \rv r} \\
    \end{align*}
\end{definition}

Building on this definition, Lemma~\ref{lemma:bregman_d_optimal} identifies that $\BregmanF$ vanishes for $d$-optimal models.

\begin{lemma}
    \label{lemma:bregman_d_optimal}
    If $\proc P$ is $d$-optimal, then $\BregmanF = 0$.
\end{lemma}
\begin{proof}
    This follows directly from $d$-optimality implying that $\fapprox{\rv \duy}$ will be identically distributed to $\fapprox{\rv y}$ given $(\rv x, \rv r)$, as $\proc D = \Dstar$.
\end{proof}

\subsection{General Decompositions of Foundation Model Error}
\label{subsec:fm_error_decomp}
In this section, we leverage the terms defined in Definitions~\ref{def:error_terms},~\ref{def:excess_terms},~and~\ref{def:bregman_interactions} to formally decompose foundation model error. We will express this decomposed error first via Lemma~\ref{lemma:model-space-err}, which shows a clear, telescoping error decomposition among our model error terms and our Bregman interaction term. However, our more ultimately useful result will be Theorem~\ref{thm:error_decomposition}, which will leverage the excess terms as well to divide the error between design-space quantities dependent on $u$, $d$, and $\proc D$ -- which can be analyzed even before a specific model $f$ is trained to motivate a foundation model design -- and our signed Bregman and excess terms, which can adjust total error in either direction, but generally cannot be fully controlled, and can only be bounded or analyzed post training or in specific contexts.

\begin{lemma}
    \label{lemma:model-space-err}
    Given a single-draw foundation model $\proc P$ using the notation/variable conventions introduced previously, then
    \[\Err{\ell}{\proc P} = \representationErr + \fDef + \dErrF + \DErrF + \BregmanF.
    \]
\end{lemma}
\begin{proof}
    Recall from Lemma~\ref{lemma:fm_error_reorg} that $\Err{\ell}{\proc P} - \representationErr = \divergence{\ell}{\phi_\rho}{\proc P}$, where $\divergence{\ell}{\phi_\rho}{\proc P} = \condDivergence{\ell}{\rv y}{\fapprox{\rv y}}{\rv x, \rv r}.$
    Further, by definition, $\BregmanF = \condDivergence{\ell}{\rv y}{\fapprox{\rv y}}{\rv x, \rv r} - \DErrF - (\dErrF + \fDef)$.
    Therefore, $\BregmanF + \DErrF + \dErrF + \fDef + \representationErr = \Err{\ell}{\proc P}$, as desired.
\end{proof}

\begin{theorem}[Foundation Model Error Decomposition]
    \label{thm:error_decomposition}
    Let $\proc P$ be a foundation model with the usual notation. Then,
    \begin{align*}
        \Err{\ell}{\proc P}
          &= \representationErr &\qquad \text{Representation Error} \\
          &\quad + \left(\uDef + \dErrU + \DErrU \right) &\qquad \text{Design Error} \\
          &\quad + \left(\excessFDeficiency + \excessdErr + \excessDErr\right) &\qquad \text{Model Excess} \\
          &\quad + \BregmanF &\qquad \text{Geometric Bregman Interaction Term.}
    \end{align*}
\end{theorem}
\begin{proof}
    This follows directly from Lemma~\ref{lemma:model-space-err} and the observation that the excess terms are precisely the difference between the model error terms used in Lemma~\ref{lemma:model-space-err} and the design-space error terms introduced here.
\end{proof}

Note that Theorem~\ref{thm:error_decomposition} does not precisely tell us anything distinct from the generic expression of error. The terms used in its decomposition are constructed precisely to enable a decomposition of this form, and the model excess terms and the Bregman interaction terms are both signed, so it is not the case that we can even say that the design-space error terms lower bound total error in general (though by Lemma~\ref{lemma:fm_error_reorg} we can say that $\representationErr$ does lower bound total error).

Nevertheless, when designing a foundation model algorithm, especially under current AI systems, we typically can only directly control the design error terms and cannot intentionally influence model excess or the Bregman interaction terms directly. E.g., if we design a foundation model with a fixed encoder $\proc E$ and decoder $\proc D$ and a target labeler $d$ and universal task $u$ in mind, then train a model $f$ to approximate $u$, it is not the case that we will be able to directly or even indirectly dictate the sign or magnitude of the model excess or Bregman interaction terms at all -- so identifying how they split from the design error terms, which we can control, allows us to meaningfully and transparently justify our design choices and their limitations. Beyond this benefit, in select settings we will be able to enforce more concrete bounds on some of these non-design terms, enabling us to bound or otherwise more richly understand the total error, which we will explore next.

\paragraph{Example} The decomposition highlights distinct limitations across foundation model paradigms. An underspecified LLM prompt can induce representation error $\representationErr$, while an autoregressive EHR model approximating task probabilities through finite Monte Carlo sampling may incur decoder error $\DErrU$. In contrast, task-conditioned models such as EveryQuery use an identity decoder, eliminating this decoder approximation error. Thus, different architectures can concentrate error in fundamentally different components of the same decomposition.

\paragraph{Practical Settings in Which Model Excess and Bregman Interaction Terms can be Bounded.}
It is not practical to insist on $u$-optimality, as this cannot be reliably achieved in settings where $f$ is learned (though we note trivially by Lemma~\ref{lemma:excess_optimal} that this zeros all excess terms). However, it may be practical to insist upon a model being both well-founded and $d$-optimal -- e.g., were you able to perfectly learn the universal task $u$, the total error would reduce to the irreducible representation error $\representationErr$. Lemma~\ref{lemma:total_error_well_founded_d_optimal} shows that in this setting, all error terms can be reduced solely to $\representationErr + \excessFDeficiency + \excessdErr = \representationErr + \condDivergence{\ell}{\rv y}{\fapprox{\rv \duy}}{\rv x, \rv r}$.

\begin{lemma}
    \label{lemma:total_error_well_founded_d_optimal}
    Let $\proc P$ be a foundation model that is well-founded in the sense of Definition~\ref{def:well-founded} and $d$-optimal in sense of Definition~\ref{def:consistent_fm}. Then
    \begin{align*}
        \Err{\ell}{\proc P} 
        &= \representationErr + \excessFDeficiency + \excessdErr \\ 
        &= \representationErr + \fDef + \dErrF \\
        &= \representationErr + \condDivergence{\ell}{\rv y}{\fapprox{\rv \duy}}{\rv x, \rv r}.
    \end{align*}
\end{lemma}
\begin{proof}
    The first equality follows directly from Definition~\ref{def:well-founded} of well-founded (which sets $\uDef = \dErrU = 0$), Lemma~\ref{lemma:excess_optimal}, which establishes that $\excessDErr = 0$, Lemma~\ref{lemma:u-d-optimal-err}, which establishes that $\DErrU = 0$, and Lemma~\ref{lemma:bregman_d_optimal}, which establishes that $\BregmanF = 0$, all given $d$-optimality. The second two equalities simply expand the definitions of these excess terms and of $\dErrF$.
\end{proof}

However, in such a setting, we can do even better by re-introducing the model error $\fModelErr$. Recall from Definition~\ref{def:error_terms} that $\fModelErr$ may be measured via a different strictly proper scoring loss $\ell_u$ than $\ell$, which can make it challenging to compare these two losses. To remedy this, Definition~\ref{def:inflation_term} first introduces the notion of a loss inflation term, $\inflationTerm$ that bounds how much a stochastic kernel can possibly inflate a loss in space $\ell_u$ to one in space $\ell$.



\begin{definition}
    \label{def:inflation_term}
    Given (1) spaces $\Ybar$ and $\UY$;
    (2) strictly proper scoring losses $\ell$ and
    $\ell_u$ over these spaces, respectively;
    and (3) a stochastic kernel
    $\lambda:\kernelover{\UY}{\Ybar}$,
    define the divergence inflation coefficient as
    \[
    \inflationTerm(\lambda)
    :=
    \inf\left\{
        C\in[0,\infty):
        \divergence{\ell}{\lambda\circ p}{\lambda\circ q}
        \leq C\divergence{\ell_u}{p}{q},
        \quad
        \forall p,q\in\ProbDists{\UY}
    \right\},
    \]
    with the convention $\inf\varnothing=\infty$.
\end{definition}

Armed with Definition~\ref{def:inflation_term} and Lemma~\ref{lemma:total_error_well_founded_d_optimal}, Theorem~\ref{thm:bounded_model_error} shows that in this well-founded, $d$-optimal setting, we can leverage this to bound $\Err{\ell}{\proc P}$ in terms of $\representationErr$ and $\fModelErr$ directly.

\begin{theorem}
    \label{thm:bounded_model_error}
    Let $\proc P$ be a well-founded, $d$-optimal foundation model, with $\ell_u$ a strictly proper scoring loss over $\UY$ in addition to the strictly proper scoring loss $\ell$ over $\Ybar$. For $\rD \in \RR_D$, let $\lambda_{\rD}: \kernelover{\UY}{\Ybar}$, $\lambda_{\rD}(z) := \delta_{d_{\rD}(z)}$, be the kernel induced by the labeler, and let $C_d := \esssup_{\rv \rD} \inflationTerm(\lambda_{\rD})$.
    
    Assume that $\fModelErr<\infty$ and $C_d<\infty$.

    Then,
    \[\Err{\ell}{\proc P} \le \representationErr + C_d \fModelErr.\]
\end{theorem}
\begin{proof}
     Let $p := u \circ \proc E$ and $q := f \circ \proc E$, both kernels $\kernelover{\X \times \RR_E}{\UY}$. Then $p$ is the conditional kernel of $\gv{\rv \uy}{\rv x, \rv \rE}$ and $q$ that of $\gv{\fapprox{\rv \uy}}{\rv x, \rv \rE}$.
 
    Note that by well-foundedness ($\uDef = \dErrU = 0$), $\rv y | \rv x = x, \rv \rE = \rE, \rv \rD = \rD$ has law $\Dstar(p(x, \rE), \rD) = \lambda_{\rD} \circ p(x, \rE)$, and by definition $\gv{\fapprox{\rv \duy}}{\rv x = x, \rv \rE = \rE, \rv \rD = \rD}$ has law $\lambda_{\rD} \circ q(x, \rE)$. Hence, by Definition~\ref{def:inflation_term},
    \begin{align*}
        \condDivergence{\ell}{\rv y}{\fapprox{\rv \duy}}{\rv x, \rv r} 
          &= \EV{\rv x, \rv \rE, \rv \rD}{
            \divergence{\ell}{\lambda_{\rD} \circ p(x, \rE)}{\lambda_{\rD} \circ q(x, \rE)}
          } \\
          &\le \EV{\rv x, \rv \rE, \rv \rD}{
            \inflationTerm(\lambda_{\rD})\, \divergence{\ell_u}{p(x, \rE)}{q(x, \rE)}
          } \\
          &\le \EV{\rv x, \rv \rE, \rv \rD}{
            \Big(\esssup_{\rv \rD} \inflationTerm(\lambda_{\rD})\Big)\, \divergence{\ell_u}{p(x, \rE)}{q(x, \rE)}
          } \\
          &= C_d\, \EV{\rv x, \rv \rE, \rv \rD}{
            \divergence{\ell_u}{p(x, \rE)}{q(x, \rE)}
          } \\
          &= C_d \fModelErr.
    \end{align*}
    Combining with Lemma~\ref{lemma:total_error_well_founded_d_optimal} gives the result.
\end{proof}

We can go one step farther than Theorem~\ref{thm:bounded_model_error} in settings where $\ell_u$ is informally equivalent to $\ell$ and are $f$-divergences. In that setting, Lemma~\ref{lemma:inflation_f_div} states that $\inflationTerm \le 1$ for all possible kernels.

\begin{lemma}
    \label{lemma:inflation_f_div}
    Suppose $\ell: \ProbDists{\Ybar} \times \Ybar \to \overline{\R}$ and $\ell_u: \ProbDists{\UY} \times \UY \to \overline{\R}$ induce the same normalized $f$-divergence over their respective families. Then for every Markov kernel $\lambda: \kernelover{\UY}{\Ybar}$, $\inflationTerm(\lambda) \le 1$.
\end{lemma}
\begin{proof}
    This follows exactly from the well-established data-processing inequality, which states that the $f$-divergence between two random variables must always be at least that formed through stochastic garbling of those random variables by the same kernel.
\end{proof}


In settings where our strictly proper scoring losses can also satisfy this property, Lemma~\ref{lemma:inflation_f_div} would thus let us offer an even tighter, more natural bound. The log loss satisfies it, as its regret divergence is the KL divergence; moreover, on label spaces with at least three elements, the KL divergence is the only divergence that is both a Bregman divergence (and hence the regret of a strictly proper scoring loss) and an $f$-divergence~\citep{amari2009alpha,jiao2014curious}, so the log loss is essentially the only case in which Lemma~\ref{lemma:inflation_f_div} applies. Theorem~\ref{thm:log_loss_bounded_error} uses this to state that if $\ell$ and $\ell_u$ are the ordinary log loss and $\proc P$ is both well-founded and $d$-optimal, then $\Err{\ell}{\proc P} \le \representationErr + \fModelErr$.

\begin{theorem}
\label{thm:log_loss_bounded_error}
    Let $\proc P$ be a well-founded, $d$-optimal foundation model and let $\ell$ and $\ell_u$ both be the ordinary log-loss. Then
    \[\Err{\ell}{\proc P} \le \representationErr + \fModelErr.\]
\end{theorem}
\begin{proof}
   The regret divergence of the log loss is the KL divergence, an $f$-divergence, so Lemma~\ref{lemma:inflation_f_div} gives $\inflationTerm(\lambda_{\rD}) \le 1$ for every $\rD$, hence $C_d \le 1$, and the result follows from Theorem~\ref{thm:bounded_model_error}.
\end{proof}

\subsection{Conclusion}
In this Section, we have shown in Theorem~\ref{thm:error_decomposition} how to decompose single-draw foundation model error into the irreducible representation error; design-error terms corresponding to universal task $u$, target labeler $d$ and realized decoder $\proc D$; excess terms capturing how the approximate model $f$ interacts with the various system components to exacerbate or ameliorate design errors; and a Bregman interaction term dictated by the geometries induced by the strictly proper scoring loss $\ell$ and distributions in use. This decomposition permits intentional, justifiable design choices when constructing foundation models and explains what aspects of the system contribute to overall error.

We further showed in Theorem~\ref{thm:bounded_model_error} that in the case of a well-founded, $d$-optimal model, we can further bound the total error by the sum of the irreducible representation error and a scaled version of the normal model error capturing how well $f$ approximates $u$, where the scaling factor is dictated by geometric properties of both the target labeler $d$ and the different losses between the model target space $\UY$ and the overall label space $\Ybar$. Next, in Section~\ref{sec:FM_recipes}, we will show how to use Definition~\ref{def:reduced_fms} to realize practical recipes to build foundation models in the real world.

\section{Recipes to Build Foundation Models}
\label{sec:FM_recipes}
When considering building foundation models, an astute reader may observe that Definition~\ref{def:foundation_model} frames foundation models in a way that is identical to how one might frame a supervised-learning, single-task model to predict $\rv y$ from the joint input domain $(\rv x, \rv t)$. This analogy would suggest that to train a foundation model, we should simply sample triples $(x, t, y)$ from $\rv x, \rv t, \rv y$ and train a model as we would normally, using all the advancements in ML/AI for supervised learning to power this endeavor. This is, of course, not possible in general because we rarely, if ever, have overt access to the ability to sample from $\rv x, \rv t, \rv y$ (though in some strange setting where this were possible, this would be a viable strategy). In general, we will not even be able to sample from $\rv x, \rv r, \rv y$, which would also permit us to learn a (reduced) foundation model.

In this section, we instead ask, powered by Definitions~\ref{def:foundation_model} or \ref{def:reduced_fms}, how we can practically construct a foundation model for a domain $\X, \T, \Ybar$ when $\T$ may be latent and $\rv t$ is unknown and cannot be sampled directly. To answer this question, we provide two recipes, that each build on the following assumptions.

Firstly, for the purposes of these recipes, we will assume that $\rho$ is complete (and in particular assume that $\T$ is naturally expressed in a tractable form, such that we can reliably use $t \in \T$ as an input to functions, etc.). As was shown in Section~\ref{sec:FM_error_sources}, this is equivalent to simply acknowledging that there is irreducible error due to $\rho$, then simply modeling over the $\X, \RR, \Ybar$ space instead of $\X, \T, \Ybar$ and considering error in that lossy domain, so we can do this without loss of generality, and it will only serve to make algebra simpler.

In addition, we assume two properties:
\begin{enumerate}
    \item There exists a $\ZZ$ valued random variable $\rv z$ (jointly distributed with all others) such that $\rv x, \rv z$ is sampleable, such that $\rv x, \rv t, \rv y, \rv z$ are related via the Bayesian DAG shown in Figure~\ref{fig:graphical_model_with_z}.
    \item There exists a known, computable labeling function $g: \T \times \ZZ \to \Ybar$ such that $\rv y = g(\rv t, \rv z)$ almost surely. Note that $\rv z \perp \rv t \mid \rv x$ by Figure~\ref{fig:graphical_model_with_z}, so $g$ recovers the true task kernel through $\eta(x, t) = \pushforward{g(t, \cdot)}{\p{\gv{\rv z}{\rv x = x}}}$.\footnote{If $\gv{\rv y}{\rv t, \rv z}$ is stochastic, all results below carry over by appending an independent seed $\xi \sim U[0,1]$ to $\rv z$.}
\end{enumerate}


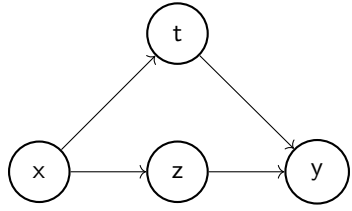
\begin{figure}
    \centering
    \begin{tikzpicture}
        \node[style=rv](x){$\rv x$};
        \node[style=rv](z)[right=of x]{$\rv z$};
        \node[style=rv](t)[above=of z]{$\rv t$};
        \node[style=rv](y)[right=of z]{$\rv y$};

        \draw[->] (x) -- (t);
        \draw[->] (t) -- (y);
        \draw[->] (x) -- (z);
        \draw[->] (z) -- (y);
    \end{tikzpicture}
    \caption{An extended graphical model where labels $\rv y$ are generated from $\rv x$ through a random variable $\rv z$.}
    \label{fig:graphical_model_with_z}
\end{figure}

\subsection{(Recipe 1) ``Generative'' Foundation Models: Model $\rv z | \rv x$, and Decode Labels Explicitly}
In Recipe 1, which we will call our ``generative foundation model recipe,''\footnote{Note that the model here is actually learning a \emph{discriminative} model over $\rv z$ in the classical sense, not a generative model, but we still call this recipe the ``generative'' recipe as this is the vehicle by which most ``generative'' foundation models today operate.} we do the following:
\begin{itemize}
    \item Define the universal task $u$ to be predicting $\rv z \mid \rv x$, with $\UX = \X$. 
    \item Let $\proc{E}(x, \cdot) = x$ simply forward the first argument and discard the task.
    \item Let $f$ be any model trained on samples of $(x, z)$ to approximate $u$.
    \item Let the decoder representation be complete, with the canonical choice
    $\rho_D(t)=\delta_t$, so that $r_D=t$ almost surely.
    \item Let the labeler be the known labeling function, $d(t, z) := g(t, z)$.
    The target decoder is then $\Dstar(q, t) = \pushforward{g(t, \cdot)}{q}$.
    \item Let the decoder $\proc{D}$ be any procedure to approximate $\Dstar$ (e.g., Monte Carlo: draw $z \sim q$ and output $g(t, z)$, up to some budget).
\end{itemize}

The autoregressive EHR foundation models discussed in Section~\ref{subsec:reduced_def} illustrate this construction, with future patient trajectories serving as the universal-task output.

Then, as is shown in Theorem~\ref{thm:recipe_1_properties}, the resulting foundation model given by $(\rho, \proc{E}, f, \proc{D})$ is well-founded in the sense of Definition~\ref{def:well-founded}, which means that the only remaining sources of error will be the excess terms, which capture error introduced by $f$'s poor approximation of $u$, the Decoder-labeler Error $\DErrU$, capturing $\proc D$'s poor approximation of $\Dstar$, and the Bregman interaction term.

\begin{theorem}
    \label{thm:recipe_1_properties}
    If a foundation model is built as in Recipe 1, then this model is well-founded.
\end{theorem}

\begin{proof}
    By Assumption~2 and $\rv z \perp \rv t \mid \rv x$, with probability one under the joint law of $(x, t)$,
    \[
        \Dstar(u(x), t)
        = \pushforward{g(t, \cdot)}{\p{\gv{\rv z}{\rv x = x}}}
        = \pushforward{g(t, \cdot)}{\p{\gv{\rv z}{\rv x = x, \rv t = t}}}
        = \eta(x, t).
    \]
    As $\rv \rD = \rv t$ almost surely, $\gv{\rv \duy}{\rv x, \rv r}$ and $\gv{\rv y}{\rv x, \rv r}$ are identically distributed, so $\condDivergence{\ell}{\rv y}{\rv \duy}{\rv x, \rv r} = 0$. Since $\rv \duy$ is feasible in the deficiency defining $\uDef$ and deficiencies are non-negative, $\uDef = 0$, and therefore $\dErrU = 0$, establishing well-foundedness.
\end{proof}

\paragraph{Advantages to Recipe 1}
This recipe has multiple advantages. Firstly, it is guaranteed to be well-founded, and thus if $f=u$ and $\proc D = \proc D^\star$, the resulting procedure achieves Bayes-optimal predictions under the stated assumptions. In addition, if $\rv x, \rv z$ are sampleable, this model requires no additional assumptions to learn. Further, note that this model is universal in the sense that \emph{any} task with a label inferrable from $\rv z$ can be inferred from this foundation model -- not just those tasks in $\T$, which will differ from the capabilities of models in Recipe 2. 


Because Recipe 1 models the full distribution of $\gv{\rv z}{\rv x}$, it can also answer \emph{conditional} queries without retraining, by rejection sampling at inference time. For instance, given tasks $t, t_1, \ldots, t_k \in \T$ and labels $y_1, \ldots, y_k \in \Y$, let $E := \{z : g(t_i, z) = y_i,\ i = 1, \ldots, k\} \subseteq \UY$. Drawing $z \sim f(x)$, discarding draws outside $E$, and labeling the accepted draws with $g(t, \cdot)$ yields $\pushforward{g(t,\cdot)}{f(x)(\cdot \mid E)}$ whenever $f(x)(E) > 0$, an approximation of the conditional label distribution $\pushforward{g(t,\cdot)}{u(x)(\cdot \mid E)}$ of task $t$ given that tasks $t_1, \ldots, t_k$ take labels $y_1, \ldots, y_k$. Recipe 2 offers no analogous mechanism unless such conditional tasks are included in $\T$ during training. A formal treatment of conditional tasks, including how the acceptance probability $u(x)(E)$ affects the error bounds of Section~\ref{sec:FM_error_sources}, is left to future work.

\paragraph{Detriments to Recipe 1}
There are also several key disadvantages to this recipe. In particular, while this property of universality is nice, it also may induce needless inefficiency. If the only tasks of interest are those in $\T$, then any additional learning effort spent on how to transform from $\rv x$ to $\rv z$ in a manner that does not benefit learning $\rv y | \rv t, \rv x$ is wasted.

In addition, while well-foundedness is nice, it is a relatively weak property. Without $d$-optimality, which these models may not have, we cannot even relate total error to underlying model error in a principled manner, meaning that there may be interactions that exacerbate or ameliorate error between $f$ and $\proc D$ that we cannot predict in advance with this theory. More generally, the lack of $d$-optimality and the requirement that $g$ be used at inference time to translate from $\rv z$ to $\rv y$ impose additional sources of error and computational costs to the user, which can be severe. 

\paragraph{Non well-founded Generative Foundation Models}
Recipe 1 is constructed such that it is well-founded, by virtue of the assumed Bayesian DAG shown in Figure~\ref{fig:graphical_model_with_z}. We could, however, generalize our notion of this recipe to apply in settings even when there is an unobserved, ``confounder'' pathway to $\rv y$ from $\rv x$ independent of $\rv z$, thereby violating the assumption that $\rv y \perp \rv x | \rv z, \rv t$. In such a setting, we could still build a model as in Recipe 1 -- let $\proc E(x, t) = x$, set $u$ to be our task of predicting $\rv z | \rv x$, set $\rhoD(t) = \delta_t$, and use $g$ as our labeler, with $\proc D$ approximating the induced target decoder. This model would still be a valid foundation model under Definition~\ref{def:foundation_model}, but it would no longer be guaranteed to be well-founded, due precisely to the error induced by ignoring the pathway $\rv x \to \rv y$ independent of $\rv z$. We will use the term ``generative foundation models'' in general to refer to any model that satisfies this form -- in particular have $\proc E(x, t) = x$, which permits an interpretation of the model in this way, and reserve \emph{well-founded generative foundation model} towards one that satisfies Recipe 1 as described fully (and is therefore well-founded).

\subsection{(Recipe 2) ``Discriminative'' Foundation Models --- Sample over a proxy distribution $q_t \in \ProbDists{\T}$, then learn $g(\rv t', \rv z) | \rv x, \rv t'$ for $\rv t' \sim q_t$}


In this recipe, which we will call the ``task-conditioned'' recipe,
we assume that we also have access to a proxy distribution
$q_t \in \ProbDists{\T}$, from which we sample $\rv t'$
independently of $(\rv x,\rv z)$. We assume that, with probability
one over the input $\rv x$, the conditional distribution of
$\rv t$ given $\rv x$ is absolutely continuous with respect to
$q_t$. In other words, for such inputs, any measurable set of
tasks assigned zero probability by $q_t$ also has zero probability
under the original conditional task distribution.

In this setting, we can sample $(x, z, t')$ from $\rv x, \rv z, \rv t'$, then compute $y' = g(t', z)$; we write $\rv y' := g(\rv t', \rv z)$ for the induced label random variable. To construct a foundation model, we then

\begin{itemize}
    \item Define the universal task $u$ to be predicting
    $\rv y' \mid \rv x, \rv t'$, with
    $\UX = \X \times \T$ and $\UY = \Ybar$.
    \item Let the encoder representation $\rho_E$ be complete (and canonically $\rho_E(\rv t) = \delta_{\rv t}$), and $E = I$ be the identity.
    \item Let $f$ be any model trained on samples $(x, t', y')$ to approximate $u$.
    \item Let the decoder representation be degenerate, the labeler be the identity, $d_{\rD}(y) = y$, and let the realized decoder simply return the output of $f$ directly: $\proc{D}(q, \cdot) = q$.
\end{itemize}

Task-conditioned models such as EveryQuery illustrate this construction, using structured task queries to directly predict downstream labels, subject to the stated assumptions on the proxy task distribution.

Then, as we will show in Theorem~\ref{thm:recipe_2_properties}, this construction yields a foundation model that is not only well-founded but also $d$-optimal, permitting us to reduce all error to model error through Theorem~\ref{thm:bounded_model_error}.

\begin{theorem}
    \label{thm:recipe_2_properties}
    If a foundation model is built as in Recipe 2, then this model will be both well-founded and $d$-optimal.
\end{theorem}
\begin{proof}
    $d$-optimality is clear, as
    $\Dstar(q, \rD) = \pushforward{\mathrm{id}}{q}
    = q = \proc D(q, \rD)$.

    Since $\rv t'$ is independent of $(\rv x,\rv z)$ and
    $\rv y'=g(\rv t',\rv z)$, the definition of $u$ gives
    \[
        u(x,t)
        = \pushforward{g(t,\cdot)}
          {\p{\gv{\rv z}{\rv x=x}}}
    \]
    with probability one when $(x,t)$ is sampled from the joint
    law of $(\rv x,\rv t')$.
    
    By the absolute continuity assumption, this equality also
    holds with probability one when $(x,t)$ is sampled from
    the original joint law of $(\rv x,\rv t)$.

    Since, by the labeling assumption, $\rv y=g(\rv t,\rv z)$ almost surely, and
    $\rv z \perp \rv t \mid \rv x$, the right-hand side
    equals $\eta(x,t)$ under that same joint law.
    Thus, $u(x,t)=\eta(x,t)$ almost surely.

    Since the task representation is complete and the encoder
    and labeler are identities,
    $\condDivergence{\ell}{\rv y}{\rv \duy}{\rv x,\rv r}=0$.
    As $\rv \duy$ is feasible in the deficiency defining
    $\uDef$ and deficiencies are non-negative, $\uDef=0$,
    and therefore $\dErrU=0$, establishing well-foundedness.
\end{proof}
\paragraph{Advantages of Recipe 2}
Unlike Recipe 1, Recipe 2 is both well-founded and $d$-optimal, meaning we can bound total error solely via model error, via Theorem~\ref{thm:bounded_model_error}. This means that we can rely on the extensive theory and practice of ordinary machine learning to learn foundation models via this recipe without concern for unmeasurable interactions between different error terms. It also has the significant benefit that there is no additional computational burden on users at inference time beyond the model $f$ itself, as $\proc D$ is trivial. It also does not waste learning energy on parts of $\rv z | \rv x$ that do not matter for any task -- though this efficiency gain may be counteracted by the distribution shift of $q_t$ discussed more below.

\paragraph{Disadvantages of Recipe 2}
The main detriment to this recipe is that this induces a source of distribution shift between training and testing, in that $\rv t' \neq \rv t$. 
At optimality under the sampling distribution, the absolute
continuity assumption transfers the almost-sure equality $f=u$
to the original joint law of $(\rv x,\rv t)$, where $u=\eta$
almost surely by the proof of
Theorem~\ref{thm:recipe_2_properties},
but in a real-world setting, this means that we cannot actually measure $\fModelErr$ in an unbiased manner without the ability to sample from $\rv t$ directly, which we assume is impossible. As such, though we can bound total error in theory, in practice we may still be left with uncertainty around what the true inference time error will be in practice, though we may be able to rely on the existing literature around distribution shift and limited ability to sample from $\rv t$ to ameliorate this in practice, which we leave as an area of future work. In addition, this recipe is only universal to $\T$ -- not to $\rv z$ as Recipe 1 is. 

Lastly, this recipe also has a source of learning inefficiency similar to Recipe 1 -- but instead of learning parts of $\rv z | \rv x$ that do not matter for any tasks, Recipe~2 may introduce learning inefficiency by requiring the model
to learn labels for proxy pairs $(x,t')$ that have very low probability
under the original joint distribution of $(\rv x,\rv t)$, since
$\rv t'$ is sampled independently of $\rv x$.

\paragraph{General Task-conditioned foundation models.}
Much like how we generalize our notion of a generative model to include any model with an unused encoder representation, we can generalize our notion of a task-conditioned model to non-well-founded settings by calling any foundation model with an unused \emph{decoder} representation a ``task-conditioned'' foundation model.

\subsection{Examples}
\label{sec:examples}

Table~\ref{tab:example_fm_decompositions} illustrates how the universal-task reduction applies to several representative foundation-model designs introduced earlier. For each example, the table gives one natural decomposition into encoder-side task representation $\rho_E$, encoder $\proc E$, universal-task input space $\UX$, predictive model $f$, universal-task output space $\UY$, decoder-side task representation $\rho_D$, and decoder $\proc D$. These decompositions are not unique: in particular, preprocessing steps such as serialization or tokenization, as well as architectural encoder or decoder components, may be absorbed into $\proc E$ or $\proc D$ depending on the level of abstraction chosen. The examples are therefore intended to illustrate how the framework can represent common foundation-model designs rather than to prescribe a unique decomposition for each architecture.

The designs in Table~\ref{tab:example_fm_decompositions} are also not intended to exhaust the possible realizations of foundation models under our definition. Prompted language models provide one broad family of examples whenever natural-language task specifications support any-task inference over the domain of interest. At the other extreme, recent decision models such as JEV expose a more explicitly task-conditioned interface: a model is provided with an input state together with one or more structured, typed questions and returns probability distributions over their possible answers. JEV is additionally trained using reinforcement learning for calibrated decisions (RLCD), which explicitly optimizes the quality and calibration of these probabilistic decisions \citep{jev}. From the perspective of our framework, such systems are naturally interpreted as instances of, or closely related to, the task-conditioned construction of Recipe~2. These emerging designs further illustrate that foundation-model behavior need not be tied to autoregressive text generation, but can instead arise from models trained to directly answer a sufficiently expressive space of task specifications.

\begin{table*}[!t]
    \centering
    \scriptsize
    \setlength{\tabcolsep}{3pt}
    \renewcommand{\arraystretch}{1.35}

    \begin{tabularx}{\textwidth}{
        @{}
        >{\raggedright\arraybackslash}p{0.14\textwidth}
        >{\raggedright\arraybackslash}X
        >{\raggedright\arraybackslash}X
        >{\raggedright\arraybackslash}X
        >{\raggedright\arraybackslash}X
        >{\raggedright\arraybackslash}X
        @{}
    }

    \toprule
    &
    \textbf{Clinical LLM}
    &
    \textbf{EHR AR}
    &
    \textbf{PFN}
    &
    \textbf{EQ}
    &
    \textbf{Multitask}
    \\
    \midrule

    $\rho_E$\newline
    {\tiny Encoder task repr.}
    &
    Natural-language prompt / instruction
    &
    $\varnothing$
    &
    Context set $\{(x_i,y_i)\}_{i=1}^n$
    &
    Structured program / query
    &
    Categorical task index
    \\

    $\proc E$\newline
    {\tiny Encoder}
    &
    Pair EHR input $x$ with prompt / instruction
    &
    Forward EHR input $x$
    &
    Pair query $x$ with context set
    &
    Pair EHR input $x$ with structured query
    &
    Pair EHR input $x$ with task index
    \\

    $\UX$ / $v$\newline
    {\tiny Universal-task input}
    &
    $\X\times\RR_E$
    &
    $\X$
    &
    $\X\times\RR_E$
    &
    $\X\times\RR_E$
    &
    $\X\times\RR_E$
    \\

    \midrule

    Training data\newline
    {\tiny Typical pretraining data}
    &
    Clinical text / EHR notes
    &
    Longitudinal structured EHR event sequences
    &
    Task episodes sampled from a prior: context set plus query $x$
    &
    EHR examples paired with structured tasks and labels
    &
    EHR examples with labels for a fixed finite task set
    \\

    $f$\newline
    {\tiny Model}
    &
    Autoregressive clinical language model
    &
    Autoregressive patient-trajectory model
    &
    PFN predictor
    &
    Task-conditioned EHR predictor
    &
    Shared multitask predictor
    \\

    $\UY$ / $z$\newline
    {\tiny Universal-task output}
    &
    Generated clinical text / answer sequence
    &
    Generated EHR trajectory
    &
    Prediction in $\Y$
    &
    Prediction in $\Y$
    &
    Prediction in $\Y$
    \\

    $\rho_D$\newline
    {\tiny Decoder task repr.}
    &
    $\varnothing$, or output schema / format
    &
    Structured task program
    &
    $\varnothing$
    &
    $\varnothing$
    &
    $\varnothing$
    \\

    $\proc D$\newline
    {\tiny Decoder}
    &
    Identity, or parse generated text into task output
    &
    Apply task program to generated trajectory distribution
    &
    Identity
    &
    Identity
    &
    Identity
    \\

    $\T$\newline
    {\tiny Task space}
    &
    Natural-language clinical tasks / instructions
    &
    EHR prediction / extraction tasks
    &
    Supervised prediction tasks / data-generating distributions
    &
    Structured EHR queries / programs
    &
    Finite discrete task set
    \\

    \midrule

    Example papers\newline
    {\tiny Representative works}
    &
    \citep{peng2023gatortrongpt}
    &
    \citep{ETHOS}
    &
    \citep{tabpfn}
    &
    \citep{everyquery}
    &
    \citep{motor}
    \\

    Recipe\newline
    {\tiny Closest Sec.~\ref{sec:FM_recipes} construction}
    &
    Hybrid: Recipe~2-like input,
    Recipe~1-like output
    &
    Recipe~1 (Generative)
    &
    Recipe~2-like
    &
    Recipe~2 (Task-conditioned)
    &
    Recipe~2-like
    \\

    \bottomrule
    \end{tabularx}

    \caption{
        Illustrative decompositions of several foundation-model paradigms
        under the universal-task reduction framework. Here, $x\in\X$
        denotes the original input, $y\in\Y$ the downstream label,
        $r_E\in\RR_E$ and $r_D\in\RR_D$ the encoder- and decoder-side
        task representations, $v\in\UX$ the universal-task input, and
        $z\in\UY$ the universal-task output. We use a minimal functional
        decomposition in which $\proc E$ only assembles the original input
        and encoder-side task representation, while $\proc D$ only performs
        the task-specific transformation needed to map universal-task
        predictions into the downstream label space. Preprocessing,
        tokenization, embeddings, and architectural encoder or decoder
        components are therefore absorbed into $f$. The ``Training data''
        and ``Example papers'' rows are included for comparison and are not
        themselves components of the formal reduction. These decompositions
        are not unique. The ``Recipe'' row indicates the closest structural
        correspondence to the constructions in Section~\ref{sec:FM_recipes} rather than an
        exact identification with a particular recipe. In particular,
        clinical LLMs are shown as hybrid constructions: task information
        enters on the encoder side through natural-language prompting, as
        in Recipe~2, while the model generates a general text output that
        may subsequently be mapped into the downstream label space, as in
        Recipe~1.
    }

    \label{tab:example_fm_decompositions}
\end{table*}

\section{Foundation Model Evaluation is Mathematically Similar to Evaluation under Algorithmic Fairness}
\label{sec:FM_evaluation}

Given a foundation model $\proc{P}$, we may also ask how we can evaluate $\proc{P}$ in a useful manner, given we care about all tasks in $\T$ in accordance with $\rv t$. In fact, this evaluation shows a striking mathematical similarity to evaluation of supervised learners in settings with constraints on algorithmic fairness (Theorem~\ref{thm:fairness_analogy}). The motivation for this theorem is, informally, that many fairness results can be seen as general statements about how a model behaves over a partitioned input domain, where the sensitive attribute defines the partition, and we can realize that in evaluation of foundation models, we likewise care about a partitioned domain, where tasks determine the partition.

\begin{theorem}
    \label{thm:fairness_analogy}
    Let $\proc{P}$ be a foundation model over extended domain $\X, \T, \Ybar$ as defined throughout.
    
    Additionally, suppose we have a distinct domain $\X', \Ybar'$, such that a model $f': \X' \to \ProbDists{\Ybar'}$ is a model of some task on that domain, with induced score $\rv s' := f'(\rv x')$. In this domain, suppose $A$ defines a set of ``attributes,'' such that for each sample $x', y'$, we can identify $a$ dictating the value of the ``attribute'' for this sample $x'$, realized through the random variables $\rv x', \rv y', \rv a$.

    Suppose one can make a logical statement $S$ or formal theoretical proof about properties of $f'$ and its evaluation as dictated across the full partition of possible values $a \in A$ -- e.g., splitting the dataset into separate partitions, each defined by the observation of a given value $a$ for $\rv a$. Then one can make an analogous statement $S$ about $\proc{P}$ over $\X, \Ybar$, such that the partitions are induced by $\T$ instead of $A$, provided that any required assumptions for said logical statement also hold over the $\X, \T, \Y$ domain with $\T$ standing in for $A$.
\end{theorem}

\begin{proof}
Such a statement is expressed entirely in terms of the joint distribution of
the prediction, label, and partitioning variable. Replacing the attribute
$\rv a$ by the task $\rv t$, the label $\rv y'$ by $\rv y$, and the score
$\rv s'$ by the foundation-model score $\rv s$ preserves the form of all
probabilities, conditional independences, and group-conditioned quantities, provided that any required assumptions for the theorem are likewise met over this space (as is an assumption in the theorem).
Hence the same argument applies with tasks indexing the partitions instead of
attributes.
\end{proof}

While basic, the utility of Theorem~\ref{thm:fairness_analogy} is that many results stemming from the literature on algorithmic fairness can now be seen to have an analogous result for general foundation models, over ``subpopulations'' defined by disparate tasks, rather than protected attributes. For example, Lemma~\ref{lemma:impossibility_thm_fms} establishes an analog of the Impossibility Theorem in fairness for foundation models:

\begin{lemma}
    \label{lemma:impossibility_thm_fms}
    Suppose for simplicity that $\Y = \{0, 1\}$. Then, let $\rv s$ be the random variable corresponding to the induced output of a foundation model $\proc{P}(\rv x, \rv r)$ (or $\proc{P}(\rv x, \rv t)$ when $\rho$ is complete), realized as a distribution-valued random variable jointly distributed with all other relevant variables, with $\rv s$ being treated as analogous to a ``score'' of the output of $\proc{P}$. In analogy to algorithmic fairness, define \emph{sufficiency} to mean that $\rv y \perp \rv t | \rv s$ and \emph{separation} to mean that $\rv s \perp \rv t | \rv y$.

    Then, provided $\rv y \not\perp \rv t$ and $\proc P$ is not a perfect predictor ($H(\rv y | \rv s) > 0$), then $\proc P$ cannot simultaneously satisfy sufficiency and separation.
    

\end{lemma}
\begin{proof}
    This is a direct application of Theorem~\ref{thm:fairness_analogy} to the well-established impossibility theorem in algorithmic fairness~\citep{classic_impossibility_thm}.
\end{proof}


\paragraph{Example} Consider a multi-task EHR model predicting outcomes with different prevalences, such as mortality and hospital readmission. Treating task identity as the partitioning variable reveals that, under the conditions of Lemma~\ref{lemma:impossibility_thm_fms}, sufficiency and separation cannot generally hold simultaneously across tasks. Thus, tradeoffs familiar from algorithmic fairness arise naturally when evaluating foundation models across heterogeneous prediction tasks.

The novelty of this section is therefore not a new impossibility theorem in
itself, but rather the observation that known fairness-style tradeoffs can be
reinterpreted as limitations of cross-task evaluation, and hence apply not
only to foundation models but to multi-task evaluation more broadly.


\bibliography{refs}
\bibliographystyle{tmlr}

\newpage
\appendix

\section{Notation and Assumptions}
\label{sec:notation}
\paragraph{Sets}
\begin{enumerate}
    \item We will use the notation $\R$ to refer to the real numbers, $\R_+$ to refer to the strictly positive real numbers, and $\overline{\R}$ to refer to the extended reals including $\pm \infty$.
\end{enumerate}

\paragraph{Typefaces / Styling}

\begin{enumerate}
    \item We will use a standard typeface ($x$) to denote ordinary variables or functions. These may be scalars or vectors or functions depending on context, as that distinction will be less essential for our purposes.
    \item We will use a calligraphic font ($\X$) to denote generic, measurable sets, which we will typically use as the domain over which we will construct random variables. 
    \item We will use a sans-serif typeface ($\rv x$) to denote random variables. As a general rule, if we define multiple random variables, they should be understood to be jointly defined, even when that joint relationship has not yet been defined.
    \item We will use the monospace typeface $\proc{E}$ to represent a ``procedure'', which is, mathematically, simply a function like any other, but we will denote it differently to reflect functions that will often be realized through more traditional computational algorithms rather than models that we might train via ML/AI methods, which we will denote with the more traditional notation $f$ or $g$.
\end{enumerate}

\paragraph{Probability}
\begin{enumerate}
    \item Given a set or boolean relationship $f:\X \to \{\text{True}, \text{False}\}$, we will use the notation $\ProbOf{\rv x}{f}$ to represent the ordinary probability that $x \in f^{-1}(\text{True})$.
    \item We will use the notation $\ProbDists{\Y}$ to denote the set of all probability density/measure functions over the set $\Y$. Note that if we state that $q \in \ProbDists{\Y}$, that does not necessarily mean that we have easy ways to perform all operations over $q$, such as sample or compute a conditional or expectation, but we will assume the ability to evaluate the probability law on given inputs.
    \item We will use the notation $x \sim \rv x$ to denote that the value $x$ is drawn from the random variable $\rv x$. At times, we may also use the notation $x \sim p$ for some $p \in \ProbDists{\X}$ to denote that $x$ is drawn from a random variable with probability density/mass function $p$. 
    \item Unless otherwise specified, we will use $\p{\rv x}$ to denote the probability mass or density function of a random variable.
    \item We will write $\mathcal L(\rv x)$ for the law (probability distribution) of a random variable $\rv x$, and $\mathcal L(\rv x \mid E)$ or $\mathcal L(\rv x \mid \rv w = w, E)$ for its conditional law given an event $E$ of positive probability (and, possibly, another random variable $\rv w$).
    \item Given a function $f$ over the domain $\X$, we will use the notation $\EV{\rv x}{f(x)}$ or $\EV{\rv x}{f}$ to denote the expectation of a function $f$ that maps samples of $\rv x$ to a measurable output space over the random variable $\rv x$. In the case that the random variable or distribution does not clearly delineate the function's input variable, we can use the expanded notation $\EV{t \sim \rv x}{f(t)}$ to indicate how the input function is applied to samples drawn from the random variable.
    \item Given a suitably nice function $f$ over a domain $\X$ and an $\X$ valued random variable $\rv x$, we will let $f(\rv x)$ denote the random variable given by applying the function $f$ to samples drawn from $\rv x$. Relatedly, if $f: \X\to \Y$ and $q \in \ProbDists{\X}$ is a probability distribution over $\X$, then we will use the pushforward notation $\pushforward{f}{q} \in \ProbDists{\Y}$ to denote the distribution obtained were one to sample from $q$ then apply $f$ (equivalent to the distribution of $f(\rv x)$ if $\rv x \sim q$).
    \item We will use a vertical bar to indicate conditional relationships, applied both to random variables (e.g., $\gv{\rv x}{\rv t = t}$ indicates the conditional random variable of $\rv x$ given the random variable $\rv t$ takes value $t \in \T$, defined over a joint distribution over $\rv x, \rv t$) or probability density functions (e.g., $\p{\gv{\rv x}{\rv t = t}}(x)$).
    \item Extending the conditional notation above to events, given $q \in \ProbDists{\Y}$ and a measurable set $S \subseteq \Y$ with $q(S) > 0$, we write $q(\cdot \mid S) \in \ProbDists{\Y}$ for $q$ conditioned on $S$, i.e., $q(A \mid S) := q(A \cap S)/q(S)$ for any measurable $A \subseteq \Y$.
    \item When it is clear from letter identity, we may simplify our conditional notation through the use of a subscript -- e.g., let $\rv x_t = \gv{\rv x}{\rv t = t}$. If we do this for multiple random variables over the same conditional, they should be understood to be considered jointly conditioned with respect to the given variable, not two independently conditioned random variables -- e.g., if we discuss $\rv x_t, \rv y_t$, this is equivalent to $\gv{\rv x, \rv y}{\rv t = t}$.
    \item When used without a target of the conditional, e.g., $\gv{\rv x}{\rv t}$, this should be understood to be a stochastic kernel from $\T$ to the space of probability distributions over $\X$, such that on input $t \in \T$, this evaluates to the typical $\gv{\rv x}{\rv t = t}$.
    \item Given an $\X$-valued random variable $\rv x$, we will let $\esssup_{\rv x}$ denote the essential supremum over $x \sim \rv x$, which we can recall is defined for any measurable function $f: \X \to \R$ to be
    \[
        \esssup_{\rv x} f := \inf \left\{ M \in \overline{\R} : \ProbOf{\rv x}{f(x) > M} = 0 \right\}.
    \]
    Informally, this is the supremum over the values that $f$ can take on $\rv x$ that will be observed with non-zero probability.
\end{enumerate}

\paragraph{Information Theory}
\label{subsec:information_theory_notation}
\begin{enumerate}
    \item We will use the traditional notation $\rho:\kernelover{\T}{\RR}$ to denote that $\rho$ defines a jointly measurable Markov kernel from $\T$ to $\RR$ (i.e., a stochastic function), such that we could introduce a random variable $\rv r_t$ for each $t \in \T$ drawn from $\rv r_t \sim \rho(t)$. Given a random variable $\rv t$ over $\T$, then we can further consider $\rho$ to induce a random variable $\rv r$ (without a subscript) over $\RR$ such that the conditional distribution given $\rv t = t$ is defined precisely by $\rho$. In this case, we can see that $\rho$ is precisely $\gv{\rv r}{\rv t}$ as introduced previously. Similarly, in this context the use of $\rv r_t$ is consistent with our simplified subscript form for conditionals discussed above. We will also use the notation $\MarkovKernels{\rv x}{\rv y}$ to represent the set of all possible Markov kernels from $\X$ to $\Ybar$ equipped with the measures induced by the given random variables. At times we may use the notation $\MarkovKernels{\X}{\Ybar}$ equivalently when the measures are clear in context.
    \item We will use the variable $\ell$ to represent a strictly proper scoring loss for $\Y$, such that $\ell: \ProbDists{\Ybar} \times \Ybar \to \overline{\R}$. We permit scoring rules and generalized entropies to take extended-real values and assume all relevant expectations and differences are well defined and finite whenever invoked.
    \item Given a proper scoring rule $\ell$ and a random variable $\rv y$ over $\Ybar$, we will let the generalized entropy $h_\ell(\rv y) = \inf_{q \in \ProbDists{\Ybar}} \EV{\rv y}{\ell(q, y)} = \EV{\rv y}{\ell(\p{\rv y}, y)}$ (where the last equality is guaranteed as $\ell$ is proper). In the event that $\ell$ is the log loss, then this is the Shannon entropy, $h_\ell = H$.
    \item To produce a generalized notion of conditional $\ell$ entropy, we simply compute the $\ell$ entropy over the conditional random variable in expectation over all possible conditioned values: $h_\ell(\gv{\rv x}{\rv t}) = \EV{\rv t}{h_\ell(\gv{\rv x}{\rv t = t})}$. Note that this conditional entropy is \emph{exactly} the score assigned by $\ell$ for a Bayes optimal predictor when predicting $\rv x$ from $\rv t$ over the joint distribution of $\rv x, \rv t$. Similarly, we define the conditional generalized $\ell$-information as $I_\ell(\rv y;\rv x\mid\rv t) := h_\ell(\gv{\rv y}{\rv t}) - h_\ell(\gv{\rv y}{\rv t,\rv x}) \ge 0.$
    \item Similarly, we can define the associated regret divergence for $\ell$ over two probability distributions $p, q \in \ProbDists{\Ybar}$ as $\divergence{\ell}{p}{q} := \EV{\rv y \sim p}{\ell(q, y) - \ell(p, y)} \ge 0$ by propriety. In the case of the log loss, this is the normal KL-divergence. At a slight risk of an abuse of notation, if $\rv y$ and $\hat{\rv y}$ are two $\Ybar$ valued random variables, we can use the random variables rather than probability laws to reflect the expected quantity: 
        $\divergence{\ell}{\rv y}{\hat{\rv y}} := \EV{\rv y}{
            \ell(\p{\hat{\rv y}}, y) - \ell(\p{\rv y}, y)
        }$.
    \item We can take the typical convention of supporting divergences between conditional kernels as reflecting an expectation over the conditioning value in a context where all variables are jointly defined: $\divergence{\ell}{\rv y | \rv x}{\hat{\rv y} | \rv x} := \EV{\rv x}{\divergence{\ell}{\rv y | \rv x = x}{\hat{\rv y} | \rv x = x}}$. We will also take as the convention that $\condDivergence{\ell}{\rv y}{\hat{\rv y}}{\rv x} := \divergence{\ell}{\rv y | \rv x}{\hat{\rv y} | \rv x}$, to simplify notation.\footnote{Note that this notation is unambiguous as we cannot take the divergence between a distribution over $\Y$ alone and a conditional kernel in $\MarkovKernels{\X}{\Y}$, which is what the interpretation of this expression as being a divergence between $(\rv y)$ and $(\rv y | \rv x)$ would imply.}
    \item As these divergences are Bregman regrets, we can also recall the three-point rule, which relates the pairwise divergences between three random variables over the same space via a (signed) Bregman interaction term. Concretely, if $\rv w$, $\rv u$, and $\rv v$ are all random variables over $\Ybar$, all jointly distributed with $\rv x$ over a potentially different space, then we can define the unconditioned and conditioned Bregman interaction terms via:
    \begin{align*}
    \bregman{\ell}{\rv w}{\rv u}{\rv v}
          &:= \divergence{\ell}{\rv w}{\rv v} - (\divergence{\ell}{\rv w}{\rv u} + \divergence{\ell}{\rv u}{\rv v}) \\
    \condBregman{\ell}{\rv w}{\rv u}{\rv v}{\rv x}
          &:= \condDivergence{\ell}{\rv w}{\rv v}{\rv x} - (\condDivergence{\ell}{\rv w}{\rv u}{\rv x} + \condDivergence{\ell}{\rv u}{\rv v}{\rv x})
    \end{align*}
    Note that in the differentiable case, these Bregman interaction terms correspond explicitly to geometric quantities that can be independently expressed, but this interpretation will not be essential for us here.
    \item We can further define the generalized mutual information for $\ell$ via:
    \begin{align*}
        I_\ell(\rv y ; \rv x)
          &= h_\ell(\rv y) - h_\ell(\gv{\rv y}{\rv x}) \\
          &= \EV{\rv x}{\divergence{\ell}{\p{\gv{\rv y}{\rv x = x}}}{ \p{\rv y}}}.
    \end{align*}
    which is the improvement in Bayes-optimal predictive loss under $\ell$ from using $\rv x$ to predict $\rv y$ vs. using $\rv y$ alone.

    \item Given two kernels, $k_1: \kernelover{\X}{\Ybar}$ and $k_2: \kernelover{\Ybar}{\ZZ}$, we will use the term ``stochastic push-forward'' of $k_1$ through $k_2$ to refer to the composition $k_2 \circ k_1: \kernelover{\X}{\ZZ}$ such that $k_2 \circ k_1: x \mapsto \EV{y \sim k_1(x)}{k_2(y)}$. We extend this notation in two ways. First, we identify a distribution $p \in \ProbDists{\Ybar}$ with the constant kernel from a one-point space, so that $k_2 \circ p := \EV{y \sim p}{k_2(y)} \in \ProbDists{\ZZ}$. Second, given a $\Ybar$-valued random variable $\rv y$, we write $k_2 \circ \rv y$ for the $\ZZ$-valued random variable obtained by drawing from $k_2(\rv y)$, jointly distributed with all others and conditionally independent of them given $\rv y$; its law is $k_2 \circ \p{\rv y}$, and, by this conditional independence, $\gv{(k_2 \circ \rv y)}{\rv x} = k_2 \circ (\gv{\rv y}{\rv x})$ for any other random variable $\rv x$. When $k_2$ is induced by a measurable function $g : \Ybar \to \ZZ$, i.e., $k_2(y) = \delta_{g(y)}$, these reduce to $\pushforward{g}{k_1(x)}$, $\pushforward{g}{p}$, and the random variable $g(\rv y)$ (with law $\pushforward{g}{\p{\rv y}}$), respectively.
    
    \item We can further define the notion of an $\ell$-deficiency, in the sense of statistical deficiencies in stochastic channels, a la~\citet{deficiency_def,deficiency_def_2}. Let $\rv x, \rv y, \rv z$ be jointly defined random variables taking values in $\X, \Ybar, \ZZ$, respectively. Note that the typical conditional distributions $\rv y | \rv x$ and $\rv z | \rv x$ form kernels $k_{x\to y}: \kernelover{\X}{\Ybar}$ such that $k_{x\to y}(x) = \p{\rv y | \rv x = x}$ and $k_{x \to z}: \kernelover{\X}{\ZZ}$ maps $x \mapsto \p{\rv z | \rv x = x}$.

    We define the $\ell$ deficiency of $\rv z$ relative to $\rv y$ to be the minimum possible error that could be obtained were one to predict $\rv y | \rv x$ through the action of any stochastic kernel in $\MarkovKernels{\ZZ}{\Ybar}$:
    \begin{align*}
    \condDeficiency{\ell}{\rv y}{\rv z}{\rv x}
        &:= \deficiency{\ell}{k_{x \to y}}{k_{x \to z}} \\
        &= \inf_{\lambda \in \MarkovKernels{\ZZ}{\Ybar}}
        \condDivergence{\ell}{\rv y}{\lambda \circ \rv z}{\rv x} \\
        &= \inf_{\lambda \in \MarkovKernels{\ZZ}{\Ybar}}
        \EV{\rv x}{
        \divergence{\ell}{\p{\gv{\rv y}{\rv x = x}}}{\lambda \circ k_{x \to z}(x)}
        }.
    \end{align*}

\end{enumerate}

\paragraph{Functions}
\begin{enumerate}
    \item We will slightly modify the traditional notation $\fnset{\X}{\Ybar}$ to denote not the set of functions from $\X$ to $\Ybar$, but instead, in cases where $\X$ and $\Ybar$ are measurable, which will nearly universally be true in this work, the set of measurable functions from $\X$ to $\Ybar$, $\mathrm{Meas}(\X, \Ybar)$. We will also use the notation $g: \X \to \Ybar$ as is typical to denote that $g \in \fnset{\X}{\Ybar}$ (again assuming measurability where applicable).
\end{enumerate}

\paragraph{Assumptions}
In general, where applicable, we will assume that sets/spaces are not only measurable, but are standard Borel spaces, and assume that random variables / probability distributions defined in the context of a known joint can be used interchangeably (e.g., we can identify a random variable with its probability distribution), that conditionals are well defined, and that statements made will be assumed to hold only almost surely (a.s.) in the probabilistic sense unless otherwise specified.

\section{Any-task vs. Zero-shot}
\label{appendix:any-task-vs-zero-shot}
We mentioned in the Introduction that we use the term ``any-task inference'' rather than ``zero-shot inference'' to avoid confusion, but here we can actually say more formally that these two requirements are very similar, and arguably identical in all domains where $\rv t$ is a continuous random variable everywhere. Concretely, in such a case, if we satisfy the two restrictions given above (can support inference for any task, no fine-tuning), then every evaluation we perform will be on tasks $\rv t$ that were almost surely\footnote{Meaning that the statement holds with probability one.} never seen during training (by continuity of $\rv t$), implying that every task we assess will be assessed in a zero-shot manner. In a setting where $\rv t$ is not continuous (e.g., $\T$ is finite), then this may not hold. However, in this setting, the notion of ``true zero-shot capabilities'' (without assuming some notion of distribution shift to a new space of task $\T$) is somewhat strange, as the space of tasks is fixed and finite, so there is nothing preventing a model from literally memorizing every single task that could ever exist directly with no risk of generalization error. Regardless, we use the term ``any-task inference'' here to focus squarely on the evaluation side, rather than on the training side, and avoid any confusion on this point.

\section{Foundation Models vs. Supervised, Single-task Models}
\label{appendix:fms_v_supervised_learning}
Definition~\ref{def:foundation_model} implies the trivial, known results that the best possible foundation model has the same Bayes-optimal solution as the best possible supervised, single-task model for a given task $t \in \T$, and thus has the same optimal error/performance. This is an obvious result, but it is worth stating explicitly to dispel any confusion that foundation models are somehow better than supervised learning models at optimality. Instead, their advantage must come in efficiency, learnability, flexibility, or along some other axis in practice.

In fact, we can go one step further than this. Suppose we have $\T' \subseteq \T$ also measurable, with $\rv t'$ defined such that $\rv t | t \in \T'$, and $\rv x, \rv y$ as usual, yielding two possible domains, $D = (\X, \T, \Y)$ and $D' = (\X, \T', \Y)$ with the usual associated random variables. Further, suppose $\mathcal F \subseteq \fnset{\X \times \T}{\ProbDists{\Y}}$ be some function class of interest within the space of all foundation models for $D$ and $\mathcal F' \subseteq \fnset{\X \times \T'}{\ProbDists{\Y}}$ be a function class of interest within the space of foundation models for $D'$. For example, suppose that these function classes are defined both by operational constraints (e.g., neural networks with no more than a certain number of parameters), but also by ``admissibility'' constraints based on performance (e.g., the resulting foundation model will only be used if it achieves a certain average performance across all tasks within the domain, and otherwise it would not be considered successful). 

In this setting, if we were to consider a foundation model over $D$ to be ``permitted'' in general, it must be within $\mathcal F$. If we were to then evaluate this model over $\T' \subseteq \T$ alone and determine if it is ``permitted'' for the restricted domain $D'$, we would make this conclusion based on whether or not it is also within $\mathcal F'$. As such, we can say that a foundation model over $D$ then considered as a solution for $D'$ would be allowed only if it lives within $\mathcal F \cap \mathcal F' \subseteq \mathcal F'$. This therefore trivially implies that were we to consider the set of all foundation models permitted under $D'$ alone vs. those permitted under $D$ and $D'$, the best possible performance of a model over $D'$ alone must be at least as good as that of one permitted under both $D$ and $D'$.

Concretely, as an example, this implies that if we train a single task model for a task $t \in \T$ under any acceptance criteria (e.g., it must obtain a certain performance and have a certain capacity limitation), and alternatively train a foundation model over $\T$, asserting that it must likewise satisfy the capacity limitation and attain a similar performance on $t$, but also satisfy some more general constraints due to its ``foundation model'' nature, then it is theoretically clear that it is impossible for the optimal such foundation model to outperform the optimal such single-task model. Instead, any performance difference we hope to achieve between single-task and foundation models must be accompanied by differences in prescribed criteria, such as permitting the foundation model to leverage additional capacity.
\end{document}